\newcommand*{\ICML}{}
\newcommand*{\CAMREADY}{}

\ifdefined\NIPS
	\PassOptionsToPackage{numbers}{natbib}
	\documentclass{article}
	\ifdefined\CAMREADY	
		\usepackage[final]{nips_2017}
	\else
		\usepackage{nips_2017}
	\fi
	\usepackage[utf8]{inputenc} 
	\usepackage[T1]{fontenc}    
	\usepackage{hyperref}       	
	\usepackage{url}            	
	\usepackage{booktabs}       
	\usepackage{amsfonts}       
	\usepackage{nicefrac}       	
	\usepackage{microtype}      
\fi
\ifdefined\CVPR
	\documentclass[10pt,twocolumn,letterpaper]{article}
	\usepackage{cvpr}
	\usepackage{times}
	\usepackage{epsfig}
	\usepackage{graphicx}
	\usepackage{amsmath}
	\usepackage{amssymb}
	\usepackage[square,comma,numbers]{natbib}
\fi
\ifdefined\AISTATS
	\documentclass[twoside]{article}
	\ifdefined\CAMREADY
		\usepackage[accepted]{aistats2015}
	\else
		\usepackage{aistats2015}
	\fi
	\usepackage{url}
	\usepackage[round,authoryear]{natbib}
\fi
\ifdefined\ICML
	\documentclass{article}
	\usepackage{microtype}
	\usepackage{graphicx}
	\usepackage{subfigure}
	\usepackage{booktabs}
	\usepackage{hyperref}
	
	\ifdefined\CAMREADY
		\usepackage[accepted]{icml2018}
	\else
		\usepackage{icml2018}
	\fi
\fi
\ifdefined\ICLR
	\documentclass{article}
	\usepackage{iclr2018_conference,times}
	\usepackage{hyperref}
	\usepackage{url}

	\ifdefined\CAMREADY
		\iclrfinalcopy
	\fi
\fi
\ifdefined\COLT
	\ifdefined\CAMREADY
		\documentclass{colt2017}
	\else
		\documentclass[anon,12pt]{colt2017}
	\fi
	\usepackage{times}
\fi

\usepackage{color}
\usepackage{amsfonts}
\usepackage{algorithm}
\usepackage{algorithmic}
\usepackage{graphicx}
\usepackage{nicefrac}
\usepackage{bbm}
\usepackage{wrapfig}
\usepackage{tikz}
\usepackage[titletoc,title]{appendix}
\usepackage{stmaryrd}
\usepackage{comment}
\usepackage{endnotes}
\usepackage{hyperendnotes}
\usepackage{enumitem}
\ifdefined\COLT
\else
	\usepackage{amsmath}
	\usepackage{amsthm}
\fi

\ifdefined\COLT
	\newtheorem{claim}[theorem]{Claim}
	\newtheorem{fact}[theorem]{Fact}
	\newcommand{\qed}{\hfill\ensuremath{\blacksquare}}
\else
	
	\newtheorem{lemma}{Lemma}
	
	\newtheorem{theorem}{Theorem}

	\newtheorem{observation}{Observation}

	\newtheorem{claim}{Claim}
	
\fi

\def\be{\begin{equation}}
\def\ee{\end{equation}}
\def\beas{\begin{eqnarray*}}
\def\eeas{\end{eqnarray*}}
\def\bea{\begin{eqnarray}}
\def\eea{\end{eqnarray}}

\newcommand{\x}{{\mathbf x}}
\newcommand{\y}{{\mathbf y}}

\newcommand{\uu}{{\mathbf u}}
\newcommand{\vv}{{\mathbf v}}
\newcommand{\w}{{\mathbf w}}

\newcommand{\e}{{\mathbf e}}

\newcommand{\0}{{\mathbf 0}}

\newcommand{\U}{{\mathcal U}}
\newcommand{\X}{{\mathcal X}}
\newcommand{\Y}{{\mathcal Y}}
\newcommand{\OO}{{\mathcal O}}

\newcommand{\EE}{\mathop{\mathbb E}} % expectation operator
\newcommand{\R}{{\mathbb R}}
\newcommand{\N}{{\mathbb N}}

\newcommand{\abs}[1]{\left\lvert#1 \right\rvert}
\newcommand{\norm}[1]{\left\|#1 \right\|}

\newcommand{\inprod}[2]  {\left\langle{#1},{#2}\right\rangle}

\definecolor{xcolor-gray}{gray}{0.95}

\ifdefined\CVPR
	\usepackage[breaklinks=true,bookmarks=false]{hyperref}
	\ifdefined\CAMREADY
		\cvprfinalcopy 
	\fi
\fi
\ifdefined\NOTESAPP
	\let\note\endnote
\else
	\let\note\footnote
\fi

\ifdefined\ICML
	\newcommand*{\ABBR}{}
\fi
\ifdefined\NIPS
	\newcommand*{\ABBR}{}
\fi
\ifdefined\ICLR
	\newcommand*{\ABBR}{}
\fi
\ifdefined\COLT
	\newcommand*{\ABBR}{}
\fi
\ifdefined\ABBR
	\newcommand{\eg}{\emph{e.g.}}
	\newcommand{\ie}{\emph{i.e.}}
	\newcommand{\cf}{\emph{cf.}}
	\newcommand{\etc}{\emph{etc.}}
	\newcommand{\vs}{\emph{vs.}}

\fi

\begin{document}

% TITLE AND AUTHORS
\ifdefined\NIPS
	\title{On the Optimization of Deep Networks: Implicit Acceleration by Overparameterization}
	\author{
	Author 1 \\
	Author 1 Institution \\
	\texttt{author1@email} \\
	\And 
	Author 2 \\
	Author 2 Institution \\
	\texttt{author2@email} \\
	\And 
	Author 3 \\
	Author 3 Institution \\
	\texttt{author3@email} \\
	}
	\maketitle
\fi
\ifdefined\CVPR
	\title{On the Optimization of Deep Networks: Implicit Acceleration by Overparameterization}
	\author{
	Author 1 \\
	Author 1 Institution \\	
	\texttt{author1@email} \\
	\and
	Author 2 \\
	Author 2 Institution \\
	\texttt{author2@email} \\	
	\and
	Author 3 \\
	Author 3 Institution \\
	\texttt{author3@email} \\
	}
	\maketitle
\fi
\ifdefined\AISTATS
	\twocolumn[
	\aistatstitle{On the Optimization of Deep Networks: Implicit Acceleration by Overparameterization}
	\ifdefined\CAMREADY
		\aistatsauthor{Author 1 \And Author 2 \And Author 3}
		\aistatsaddress{Author 1 Institution \And Author 2 Institution \And Author 3 Institution}
	\else
		\aistatsauthor{Anonymous Author 1 \And Anonymous Author 2 \And Anonymous Author 3}
		\aistatsaddress{Unknown Institution 1 \And Unknown Institution 2 \And Unknown Institution 3}
	\fi
	]	
\fi
\ifdefined\ICML
	\twocolumn[
	\icmltitlerunning{On the Optimization of Deep Networks: Implicit Acceleration by Overparameterization}
	\icmltitle{On the Optimization of Deep Networks: \\ Implicit Acceleration by Overparameterization} 
	\icmlsetsymbol{equal}{*}
	\begin{icmlauthorlist}
	\icmlauthor{Sanjeev Arora}{PU,IAS} % Add ''equal'' next to institution identifier if appropriate
	\icmlauthor{Nadav Cohen}{IAS}
	\icmlauthor{Elad Hazan}{PU,Google}
	\end{icmlauthorlist}
	\icmlaffiliation{PU}{Department of Computer Science, Princeton University, Princeton, NJ, USA}
	\icmlaffiliation{IAS}{School of Mathematics, Institute for Advanced Study, Princeton, NJ, USA}
	\icmlaffiliation{Google}{Google Brain, USA}	
	\icmlcorrespondingauthor{Nadav Cohen}{cohennadav@ias.edu}
	\icmlkeywords{Deep Learning, Learning Theory, Non-Convex Optimization}
	\vskip 0.3in
	]
	\printAffiliationsAndNotice{} % Add \icmlEqualContribution inside {} if appropriate
\fi
\ifdefined\ICLR
	\title{On the Optimization of Deep Networks: Implicit Acceleration by Overparameterization}
	\author{
	Author 1 \\
	Author 1 Institution \\
	\texttt{author1@email}
	\And
	Author 2 \\
	Author 2 Institution \\
	\texttt{author2@email}
	\And
	Author 3 \\ 
	Author 3 Institution \\
	\texttt{author3@email}
	}
	\maketitle
\fi
\ifdefined\COLT
	\title{On the Optimization of Deep Networks: Implicit Acceleration by Overparameterization}
	\coltauthor{
	\Name{Author 1} \Email{author1@email} \\
	\addr Author 1 Institution
	\And
	\Name{Author 2} \Email{author2@email} \\
	\addr Author 2 Institution
	\And
	\Name{Author 3} \Email{author3@email} \\
	\addr Author 3 Institution}
	\maketitle
\fi

% ABSTRACT
\begin{abstract}

Conventional wisdom in deep learning states that increasing depth improves expressiveness but complicates optimization.
This paper suggests that, sometimes, increasing depth can speed up optimization.
The effect of depth on optimization is decoupled from expressiveness by focusing on settings where additional layers amount to overparameterization~--~linear neural networks, a well-studied model.
Theoretical analysis, as well as experiments, show that here depth acts as a preconditioner which may accelerate convergence.
Even on simple convex problems such as linear regression with $\ell_p$ loss, $p>2$, gradient descent can benefit from transitioning to a non-convex overparameterized objective, more than it would from some common acceleration schemes.
We also prove that it is mathematically impossible to obtain the acceleration effect of overparametrization via gradients of any regularizer.

\end{abstract}

% KEYWORDS
\ifdefined\COLT
	\medskip
\fi

% INTRODUCTION
\section{Introduction} \label{sec:intro}

How does depth help?
This central question of deep learning still eludes full theoretical understanding.
The general consensus is that there is a trade-off: increasing depth improves expressiveness, but complicates optimization. 
Superior expressiveness of deeper networks, long suspected, is now confirmed by theory, albeit for fairly limited learning problems~\cite{eldan2015power,raghu2016expressive,lee2017ability,cohen2017analysis,daniely2017depth,arora2018understanding}.
Difficulties in optimizing deeper networks have also been long clear~--~the signal held by a gradient gets buried as it propagates through many layers. 
This is known as the ``vanishing/exploding gradient problem''.
Modern techniques such as batch normalization~\cite{ioffe2015batch} and residual connections~\cite{he2015deep} have somewhat alleviated these difficulties in practice.

Given the longstanding consensus on expressiveness \vs~optimization trade-offs, this paper conveys a rather counterintuitive message: increasing depth can \emph{accelerate} optimization.
The effect is shown, via first-cut theoretical and empirical analyses, to resemble a combination of two well-known tools in the field of optimization: 
\emph{momentum}, which led to provable acceleration bounds~\cite{nesterov1983method}; 
and \emph{adaptive regularization}, a more recent technique proven to accelerate by \citet{duchi2011adaptive} in their proposal of the AdaGrad algorithm.
Explicit mergers of both techniques are quite popular in deep learning~\cite{kingma2014adam,tieleman2012lecture}.
It is thus intriguing that merely introducing depth, with no other modification, can have a similar effect,~but \emph{implicitly}.

There is an obvious hurdle in isolating the effect of depth on optimization: if increasing depth leads to faster training on a given dataset, how can one tell whether the improvement arose from a true acceleration phenomenon, or simply due to better representational power (the shallower network was unable to attain the same training loss)? 
We respond to this hurdle by focusing on \emph{linear neural networks} (\cf~\citet{saxe2013exact,goodfellow2016deep,hardt2016identity,kawaguchi2016deep}).
With these models, adding layers does not alter expressiveness; it manifests itself only in the replacement of a matrix parameter by a product of matrices~--~an \emph{overparameterization}. 

We provide a new analysis of linear neural network optimization via direct treatment of the differential equations associated with gradient descent when training arbitrarily deep networks on arbitrary loss functions.
We find that the overparameterization introduced by depth leads gradient descent to operate as if it were training a shallow (single layer) network, while employing a particular preconditioning scheme. 
The preconditioning promotes movement along directions already taken by the optimization, and can be seen as an acceleration procedure that combines momentum with adaptive learning rates.
Even on simple convex problems such as linear regression with $\ell_p$~loss, $p>2$, overparameterization via depth can significantly speed up training.
Surprisingly, in some of our experiments, not only did overparameterization outperform na\"ive gradient descent, but it was also faster than two well-known acceleration methods~--~AdaGrad~\cite{duchi2011adaptive} and AdaDelta~\cite{zeiler2012adadelta}. 
In addition to purely linear networks, we also demonstrate (empirically) the implicit acceleration of overparameterization on a non-linear model, by replacing hidden layers with depth-$2$ linear networks. 
The implicit acceleration of overparametrization is different from standard regularization~--~we prove its effect cannot be attained via gradients of \emph{any} fixed regularizer.

Both our theoretical analysis and our empirical evaluation indicate that acceleration via overparameterization need not be computationally expensive.
From an optimization perspective, overparameterizing using wide or narrow networks has the same effect~--~it is only the depth that matters.

The remainder of the paper is organized as follows.
In Section~\ref{sec:related} we review related work.
Section~\ref{sec:warmup} presents a warmup example of linear regression with $\ell_p$~loss, demonstrating the immense effect overparameterization can have on optimization, with as little as a single additional scalar.
Our theoretical analysis begins in Section~\ref{sec:lnn}, with a setup of preliminary notation and terminology.
Section~\ref{sec:dynamics} derives the preconditioning scheme implicitly induced by overparameterization, followed by Section~\ref{sec:impossible} which shows that this form of preconditioning is not attainable via any regularizer.
In Section~\ref{sec:acceleration} we qualitatively analyze a very simple learning problem, demonstrating how the preconditioning can speed up optimization.
Our empirical evaluation is delivered in Section~\ref{sec:exp}.
Finally, Section~\ref{sec:conc} concludes.

% RELATED WORK
\section{Related Work} \label{sec:related}

Theoretical study of optimization in deep learning is a highly active area of research.
Works along this line typically analyze critical points (local minima, saddles) in the landscape of the training objective, either for linear networks (see for example \citet{kawaguchi2016deep,hardt2016identity} or \citet{baldi1989neural} for a classic account), or for specific non-linear networks under different restrictive assumptions (\cf~\citet{choromanska2015loss,Haeffele:2015vz,soudry2016no,safran2017spurious}).
Other works characterize other aspects of objective landscapes, for example~\citet{safran2016quality} showed that under certain conditions a monotonically descending path from initialization to global optimum exists (in compliance with the empirical observations of~\citet{goodfellow2014qualitatively}).

The dynamics of optimization was studied in~\citet{fukumizu1998effect} and~\citet{saxe2013exact}, for linear networks.
Like ours, these works analyze gradient descent through its corresponding differential equations.
\citet{fukumizu1998effect}~focuses on linear regression with~$\ell_2$ loss, and does not consider the effect of varying depth~--~only a two (single hidden) layer network is analyzed.
\citet{saxe2013exact}~also focuses on $\ell_2$~regression, but considers any depth beyond two (inclusive), ultimately concluding that increasing depth can \emph{slow down} optimization, albeit by a modest amount.
In contrast to these two works, our analysis applies to a general loss function, and any depth including one.
Intriguingly, we find that for $\ell_p$~regression, acceleration by depth is revealed only when~$p>2$.
This explains why the conclusion reached in~\citet{saxe2013exact} differs from ours.

Turning to general optimization, accelerated gradient (momentum) methods were introduced in~\citet{nesterov1983method}, and later studied in numerous works (see~\citet{wibisono2016variational} for a short review).
Such methods effectively accumulate gradients throughout the entire optimization path, using the collected history to determine the step at a current point in time.
Use of preconditioners to speed up optimization is also a well-known technique.
Indeed, the classic Newton's method can be seen as preconditioning based on second derivatives. 
Adaptive preconditioning with only first-order (gradient) information was popularized by the BFGS algorithm and its variants (\cf~\citet{jorge}).
Relevant theoretical guarantees, in the context of regret minimization, were given in~\citet{HAK07,duchi2011adaptive}.  
In terms of combining momentum and adaptive preconditioning, Adam~\cite{kingma2014adam} is a popular approach, particularly for optimization of deep networks.

Algorithms with certain theoretical guarantees for non-convex optimization, and in particular for training deep neural networks, were recently suggested in various works, for example \citet{ge2015escaping,agarwal2017finding,carmon2016accelerated,Janzamin:2015uz,livni2013algorithm} and references therein. 
Since the focus of this paper lies on the analysis of algorithms already used by practitioners, such works lie outside our scope.

% WARMUP: L_P REGRESSION
\section{Warmup: $\ell_p$ Regression} \label{sec:warmup}

We begin with a simple yet striking example of the effect being studied.
For linear regression with $\ell_p$~loss, we will see how even the slightest overparameterization can have an immense effect on optimization.
Specifically, we will see that simple gradient descent on an objective overparameterized by a single scalar, corresponds to a form of accelerated gradient descent on the original objective.

Consider the objective for a scalar linear regression problem with~$\ell_p$ loss ($p$~--~even positive integer):
\vspace{-2mm}
$$L(\w)=\EE\nolimits_{(\x,y)\sim{S}}\Big[\tfrac{1}{p}(\x^\top\w-y)^p\Big]$$
\vspace{-5mm}\\
$\x\in\R^d$ here are instances, $y\in\R$ are continuous labels, $S$~is a finite collection of labeled instances (training set), and $\w\in\R^d$ is a learned parameter vector.
Suppose now that we apply a simple overparameterization, replacing the parameter vector~$\w$ by a vector~$\w_1\in\R^d$ times a scalar~$w_2\in\R$:
\vspace{-2mm}
$$L(\w_1,w_2)=\EE\nolimits_{(\x,y)\sim{S}}\Big[\tfrac{1}{p}(\x^\top\w_1{w}_2-y)^p\Big]$$
\vspace{-5mm}\\
Obviously the overparameterization does not affect the expressiveness of the linear model.
How does it affect optimization?
What happens to gradient descent on this non-convex objective? 

\begin{observation}
Gradient descent over~$L(\w_1,w_2)$, with fixed small learning rate and near-zero initialization, is equivalent to gradient descent over~$L(\w)$ with particular adaptive learning rate and momentum terms.
\end{observation}

To see this, consider the gradients of~$L(\w)$ and\,$L(\w_1,w_2)$:
\vspace{-1mm}
\beas
\nabla_\w~~&:=&\EE\nolimits_{(\x,y)\sim{S}}\big[(\x^\top\w-y)^{p-1}\x\big] \\
\nabla_{\w_1}&:=&\EE\nolimits_{(\x,y)\sim{S}}\big[(\x^\top\w_1{w}_2-y)^{p-1}w_2\x\big] \\
\nabla_{w_2}&:=&\EE\nolimits_{(\x,y)\sim{S}}\big[(\x^\top\w_1{w}_2-y)^{p-1}\w_1^\top\x\big]
\eeas
\vspace{-5mm}\\
Gradient descent over~$L(\w_1,w_2)$ with learning rate~$\eta>0$:
\vspace{-1mm}
$$\w_1^{(t+1)}\mapsfrom\w_1^{(t)}{-}\eta\nabla_{\w_1^{(t)}}
\quad,\quad
w_2^{(t+1)}\mapsfrom{w}_2^{(t)}{-}\eta\nabla_{w_2^{(t)}}$$
\vspace{-5mm}\\
The dynamics of the underlying parameter $\w=\w_{1}w_2$ are:
\vspace{-1mm}
\beas
\w^{(t+1)}=\w_1^{(t+1)}w_2^{(t+1)}
\qquad\qquad\qquad\qquad\qquad\qquad\qquad\\
\mapsfrom(\w_1^{(t)}{-}\eta\nabla_{\w_1^{(t)}})(w_2^{(t)}{-}\eta\nabla_{w_2^{(t)}}) 
~~~\quad\qquad\qquad\qquad\qquad\\
=\w_1^{(t)}w_2^{(t)}-\eta{w}_2^{(t)}\nabla_{\w_1^{(t)}}-\eta\nabla_{w_2^{(t)}}\w_1^{(t)}+\OO(\eta^2)
\,~~~~\qquad\\
=\w^{(t)}-\eta(w_2^{(t)})^2\nabla_{\w^{(t)}}-\eta(w_2^{(t)})^{-1}\nabla_{w_2^{(t)}}\w^{(t)}+\OO(\eta^2)
\,
\eeas
\vspace{-5mm}\\
$\eta$ is assumed to be small, thus we neglect~$\OO(\eta^2)$.
Denoting $\rho^{(t)}{:=}\eta(w_2^{(t)})^2\,{\in}\,\R$ and $\gamma^{(t)}{:=}\eta(w_2^{(t)})^{-1}\nabla_{w_2^{(t)}}\,{\in}\,\R$, this gives:
\vspace{-2mm}
$$\w^{(t+1)}\mapsfrom\w^{(t)}-\rho^{(t)}\nabla_{\w^{(t)}}-\gamma^{(t)}\w^{(t)}$$
\vspace{-6mm}\\
Since by assumption $\w_1$ and~$w_2$ are initialized near zero, $\w$~will initialize near zero as well.
This implies that at every iteration~$t$, $\w^{(t)}$~is a weighted combination of past gradients.
There thus exist $\mu^{(t,\tau)}\in\R$ such that:
\vspace{-2mm}
$$\w^{(t+1)}\mapsfrom\w^{(t)}-\rho^{(t)}\nabla_{\w^{(t)}}-\sum\nolimits_{\tau=1}^{t-1}\mu^{(t,\tau)}\nabla_{\w^{(\tau)}}$$
\vspace{-5mm}\\
We conclude that the dynamics governing the underlying parameter~$\w$ correspond to gradient descent with a momentum term, where both the learning rate~($\rho^{(t)}$) and momentum coefficients~($\mu^{(t,\tau)}$) are time-varying and adaptive.

% LINEAR NEURAL NETWORKS
\section{Linear Neural Networks} \label{sec:lnn}

Let~$\X:=\R^d$ be a space of objects (\eg~images or word embeddings) that we would like to infer something about, and let~$\Y:=\R^k$ be the space of possible inferences.
Suppose we are given a training set $\{(\x^{(i)},\y^{(i)})\}_{i=1}^{m}\subset\X\times\Y$, along with a (point-wise) loss function~$l:\Y\times\Y\to\R_{\geq0}$.
For example, $\y^{(i)}$~could hold continuous values with~$l(\cdot)$ being the $\ell_2$~loss: $l(\hat{\y},\y)=\frac{1}{2}\norm{\hat{\y}-\y}_2^2$; or it could hold one-hot vectors representing categories with~$l(\cdot)$ being the softmax-cross-entropy loss: $l(\hat{\y},\y)=-\sum_{r=1}^{k}y_r\log(e^{\hat{y}_r}/\sum_{r'=1}^{k}e^{\hat{y}_{r'}})$, where $y_r$ and~$\hat{y}_r$ stand for coordinate~$r$ of~$\y$ and~$\hat{\y}$ respectively.
For a predictor~$\phi$, \ie~a mapping from~$\X$ to~$\Y$, the overall training loss is~$L(\phi):=\frac{1}{m}\sum_{i=1}^{m}l(\phi(\x^{(i)}),\y^{(i)})$.
If~$\phi$ comes from some parametric family~$\Phi:=\{\phi_\theta:\X\to\Y |\theta\in\Theta\}$, we view the corresponding training loss as a function of the parameters, \ie~we consider $L^\Phi(\theta):=\frac{1}{m}\sum_{i=1}^{m}l(\phi_\theta(\x^{(i)}),\y^{(i)})$.
For example, if the parametric family in question is the class of (directly parameterized) linear predictors: 
\vspace{-2mm}
\be
\Phi^{lin}:=\{\x\mapsto{W\x}|W\in\R^{k,d}\}
\label{eq:lin}
\ee
\vspace{-6mm}\\
the respective training loss is a function from~$\R^{k,d}$ to~$\R_{\geq0}$.
\vspace{1mm}

In our context, a depth-$N$ ($N\geq2$) linear neural network, with hidden widths~$n_1,n_2,\ldots,n_{N-1}{\in}\N$, is the following parametric family of linear predictors: $\Phi^{n_1{\ldots}n_{N-1}}:=\left\{\x\mapsto{W_{N}W_{N-1}{\cdots}W_1\x}|W_j{\in}\R^{n_j,n_{j-1}},j{=}1...N\right\}$,~where by definition $n_0:=d$ and~$n_N:=k$.
As customary, we refer to each~$W_j$, $j{=}1...N$, as the weight matrix of layer~$j$.
For simplicity of presentation, we hereinafter omit from our notation the hidden widths $n_1...n_{N-1}$, and simply write~$\Phi^N$ instead of~$\Phi^{n_1{\ldots}n_{N-1}}$ ($n_1{\ldots}n_{N-1}$ will be specified explicitly if not clear by context).
That is, we denote:
\vspace{-2mm}
\bea
&\Phi^N:= & 
\label{eq:lnn} \\
&\left\{\x\mapsto{W_{N}W_{N-1}{\cdots}W_1\x}|~W_j\in\R^{n_j,n_{j-1}},~j{=}1...N\right\}&
\nonumber
\eea
\vspace{-6mm}\\
For completeness, we regard a depth-$1$ network as the family of directly parameterized linear predictors, \ie~we set~$\Phi^1:=\Phi^{lin}$  (see Equation~\ref{eq:lin}).

The training loss that corresponds to a depth-$N$ linear network~--~$L^{\Phi^N}(W_1,...,W_N)$, is a function from $\R^{n_1,n_0}{\times}{\cdots}{\times}\R^{n_N,n_{N-1}}$ to~$\R_{\geq0}$.
For brevity, we will denote this function by~$L^{N}(\cdot)$.
Our focus lies on the behavior of gradient descent when minimizing~$L^{N}(\cdot)$.
More specifically, we are interested in the dependence of this behavior on~$N$, and in particular, in the possibility of increasing~$N$ leading to acceleration.
Notice that for any~$N\geq2$ we have:
\vspace{-2mm}
\be
L^N(W_1,...,W_N)=L^1(W_{N}W_{N-1}{\cdots}W_1)
\label{eq:lnn_loss_oprm}
\ee
and so the sole difference between the training loss of a depth-$N$ network and that of a depth-$1$ network (classic linear model) lies in the replacement of a matrix parameter by a product of~$N$ matrices.
This implies that if increasing~$N$ can indeed accelerate convergence, it is not an outcome of any phenomenon other than favorable properties of depth-induced overparameterization for optimization.

% IMPLICIT DYNAMICS OF GRADIENT DESCENT
\vspace{1mm}
\section{Implicit Dynamics of Gradient Descent} \label{sec:dynamics}

In this section we present a new result for linear neural networks, tying the dynamics of gradient descent on~$L^N(\cdot)$~--~the training loss corresponding to a depth-$N$ network, to those on~$L^1(\cdot)$~--~training loss of a depth-$1$ network (classic linear model).
Specifically, we show that gradient descent on~$L^N(\cdot)$, a complicated and seemingly pointless overparameterization, can be directly rewritten as a particular preconditioning scheme over gradient descent on~$L^1(\cdot)$.

When applied to~$L^N(\cdot)$, gradient descent takes on the following form:
\vspace{-2mm}
\bea
W_j^{(t+1)} \mapsfrom (1-\eta\lambda)W_j^{(t)}-\eta\frac{\partial{L^N}}{\partial{W_j}}(W_1^{(t)},\ldots,W_N^{(t)})&&
\label{eq:Wj_gd}\\[-0.5mm]
,~j=1{\ldots}N&&
\nonumber
\eea
\vspace{-5mm}\\
$\eta>0$~here is a learning rate, and $\lambda\geq0$~is an optional weight decay coefficient.
For simplicity, we regard both~$\eta$ and~$\lambda$ as fixed (no dependence on~$t$).
Define the underlying \emph{end-to-end weight matrix}:
\be
W_e:=W_{N}W_{N-1}\cdots{W}_1
\label{eq:We}
\ee
\vspace{-4mm}\\
Given that $L^N(W_1,\ldots,W_N)=L^1(W_e)$ (Equation~\ref{eq:lnn_loss_oprm}), we view~$W_e$ as an optimized weight matrix for~$L^1(\cdot)$, whose dynamics are governed by Equation~\ref{eq:Wj_gd}.
Our interest then boils down to the study of these dynamics for different choices of~$N$.
For~$N=1$ they are (trivially) equivalent to standard gradient descent over~$L^1(\cdot)$.
We will characterize the dynamics for~$N\geq2$.

To be able to derive, in our general setting, an explicit update rule for the end-to-end weight matrix~$W_e$~(Equation~\ref{eq:We}), we introduce an assumption by which the learning rate is small, \ie~$\eta^2\approx0$.
Formally, this amounts to translating Equation~\ref{eq:Wj_gd} to the following set of differential equations:
\vspace{-1mm}
\bea
\dot{W}_j(t)=-\eta\lambda{W}_j(t)-\eta\frac{\partial{L^N}}{\partial{W_j}}(W_1(t),\ldots,W_N(t))&&
\label{eq:Wj_gf}\\[-0.5mm]
,~j=1{\ldots}N&&
\nonumber
\eea
\vspace{-5mm}\\
where~$t$ is now a continuous time index, and~$\dot{W}_j(t)$ stands for the derivative of~$W_j$ with respect to time.
The use of differential equations, for both theoretical analysis and algorithm design, has a long and rich history in optimization research (see~\citet{helmke2012optimization} for an overview).
When step sizes (learning rates) are taken to be small, trajectories of discrete optimization algorithms converge to smooth curves modeled by continuous-time differential equations, paving way to the well-established theory of the latter (\cf~\citet{boyce1969elementary}).
This approach has led to numerous interesting findings, including recent results in the context of acceleration methods~(\eg~\citet{su2014differential,wibisono2016variational}).

With the continuous formulation in place, we turn to express the dynamics of the end-to-end matrix~$W_e$:
\begin{theorem}
\label{theorem:We_gf}
Assume the weight matrices~$W_1{\ldots}W_N$ follow the dynamics of continuous gradient descent (Equation~\ref{eq:Wj_gf}).
Assume also that their initial values (time~$t_0$) satisfy, for $j=1{\ldots}N-1$:
\bea
W_{j+1}^\top(t_0)W_{j+1}(t_0)=W_j(t_0)W^\top_j(t_0)&& 
\label{eq:Wj_agree}
\eea
Then, the end-to-end weight matrix~$W_e$ (Equation~\ref{eq:We}) is governed by the following differential equation:
\bea
\dot{W}_e(t)=-\eta\lambda{N}\cdot{W}_e(t) 
\quad\quad\qquad\qquad\qquad\qquad\qquad
\label{eq:We_gf}\\
-\eta\sum\nolimits_{j=1}^N\left[W_e(t)W_e^\top(t)\right]^\frac{j-1}{N}\cdot
\quad\qquad\qquad
\nonumber\\[-2mm]
\tfrac{dL^1}{dW}(W_e(t))\cdot\left[W_e^\top(t)W_e(t)\right]^\frac{N-j}{N}
\nonumber
\eea
where $[\cdot]^\frac{j-1}{N}$~and~$[\cdot]^\frac{N-j}{N}$, $j=1\ldots{N}$, are fractional power operators defined over positive semidefinite matrices.
\end{theorem}
\begin{proof}(sketch~--~full details in Appendix~\ref{app:proofs:We_gf})
If $\lambda\,{=}\,0$ (no weight decay) then one can easily show that $W_{j+1}^\top(t)\dot{W}_{j+1}(t)=\dot{W}_j(t)W_j^\top(t)$ throughout optimization.
Taking the transpose of this equation and adding to itself, followed by integration over time, imply that the difference between $W_{j+1}^\top(t)W_{j+1}(t)$ and~$W_j(t)W_j^\top(t)$ is constant.
This difference is zero at initialization (Equation~\ref{eq:Wj_agree}), thus will remain zero throughout,~\ie:
\vspace{-1mm}
\be
W_{j+1}^\top(t)W_{j+1}(t)=W_j(t)W_j^\top(t)\quad,~\forall{t\geq{t_0}}
\label{eq:Wj_agree_throughout}
\ee
\vspace{-5mm}\\
A slightly more delicate treatment shows that this is true even if~$\lambda>0$, \ie~with weight decay included.

Equation~\ref{eq:Wj_agree_throughout} implies alignment of the (left and right) singular spaces of~$W_j(t)$ and~$W_{j+1}(t)$, simplifying the product~$W_{j+1}(t)W_j(t)$.
Successive application of this simplification allows a clean computation for the product of all layers (that is,~$W_e$), leading to the explicit form presented in theorem statement (Equation~\ref{eq:We_gf}).
\end{proof}

Translating the continuous dynamics of Equation~\ref{eq:We_gf} back to discrete time, we obtain the sought-after update rule for the end-to-end weight matrix:
\vspace{-1mm}
\bea
W_e^{(t+1)}\mapsfrom(1-\eta\lambda{N})W_e^{(t)}
\quad\qquad\qquad\qquad\qquad\qquad
\label{eq:We_gd}\\
-\eta\sum\nolimits_{j=1}^N\left[W_e^{(t)}(W_e^{(t)})^\top\right]^\frac{j-1}{N}\cdot
\quad\quad\qquad
\nonumber\\[-2mm]
\tfrac{dL^1}{dW}(W_e^{(t)})\cdot\left[(W_e^{(t)})^\top{W}_e^{(t)}\right]^\frac{N-j}{N}
\nonumber
\eea
\vspace{-5mm}\\
This update rule relies on two assumptions:
first, that the learning rate~$\eta$ is small enough for discrete updates to approximate continuous ones;
and second, that weights are initialized on par with Equation~\ref{eq:Wj_agree}, which will approximately be the case if initialization values are close enough to zero.
It is customary in deep learning for both learning rate and weight initializations to be small, but nonetheless above assumptions are only met to a certain extent.
We support their applicability by showing empirically (Section~\ref{sec:exp}) that the end-to-end update rule (Equation~\ref{eq:We_gd}) indeed provides an accurate description for the dynamics of~$W_e$.

A close look at Equation~\ref{eq:We_gd} reveals that the dynamics of the end-to-end weight matrix~$W_e$ are similar to gradient descent over~$L^1(\cdot)$~--~training loss corresponding to a depth-$1$ network (classic linear model).
The only difference (besides the scaling by~$N$ of the weight decay coefficient~$\lambda$) is that the gradient~$\frac{dL^1}{dW}(W_e)$ is subject to a transformation before being used.
Namely, for~$j=1{\ldots}N$, it is multiplied from the left by~$[W_{e}W_e^\top]^\frac{j-1}{N}$ and from the right by~$[W_e^\top{W}_e]^\frac{N-j}{N}$, followed by summation over~$j$.
Clearly, when~$N=1$ (depth-$1$ network) this transformation reduces to identity, and as expected, $W_e$~precisely adheres to gradient descent over~$L^1(\cdot)$.
When~$N\geq2$ the dynamics of~$W_e$ are less interpretable.
We arrange it as a vector to gain more insight:
\begin{claim}
\label{claim:We_gd_vec}
For an arbitrary matrix~$A$, denote by~$vec(A)$ its arrangement as a vector in column-first order.
Then, the end-to-end update rule in Equation~\ref{eq:We_gd} can be written as:
\vspace{-0.5mm}
\bea
vec(W_e^{(t+1)})\mapsfrom(1-\eta\lambda{N})\cdot{vec}(W_e^{(t)}) 
\qquad\qquad
\label{eq:We_gd_vec}\\
-\eta\cdot{P}_{W_e^{(t)}}vec\left(\tfrac{dL^1}{dW}(W_e^{(t)})\right)
\nonumber
\eea
\vspace{-4mm}\\
where~$P_{W_e^{(t)}}$ is a positive semidefinite preconditioning matrix that depends on~$W_e^{(t)}$.
Namely, if we denote the singular values of~$W_e^{(t)}\in\R^{k,d}$ by $\sigma_1\ldots\sigma_{\max\{k,d\}}\in\R_{\geq0}$ (by definition $\sigma_r=0$ if $r>\min\{k,d\}$), and corresponding left and right singular vectors by $\uu_1\ldots\uu_k\in\R^k$ and $\vv_1\ldots\vv_d\in\R^d$ respectively, the eigenvectors of~$P_{W_e^{(t)}}$ are:
\vspace{-1mm}
$$vec(\uu_r\vv_{r'}^\top)\quad,r=1\ldots{k}~,~r'=1\ldots{d}$$
\vspace{-5mm}\\
with corresponding eigenvalues:
\vspace{-1mm}
$$\sum\nolimits_{j=1}^{N}\sigma_r^{2\frac{N-j}{N}}\sigma_{r'}^{2\frac{j-1}{N}}\quad,r=1\ldots{k}~,~r'=1\ldots{d}$$
\vspace{-7mm}\\
\end{claim}
\vspace{-5mm}
\begin{proof}
The result readily follows from the properties of the Kronecker product~--~see Appendix~\ref{app:proofs:We_gd_vec} for details.
\end{proof}

Claim~\ref{claim:We_gd_vec} implies that in the end-to-end update rule of Equation~\ref{eq:We_gd}, the transformation applied to the gradient~$\frac{dL^1}{dW}(W_e)$ is essentially a preconditioning, whose eigendirections and eigenvalues depend on the singular value decomposition of~$W_e$.
The eigendirections are the rank-$1$ matrices~$\uu_r\vv_{r'}^\top$, where~$\uu_r$ and~$\vv_{r'}$ are left and right (respectively) singular vectors of~$W_e$.
The eigenvalue of~$\uu_r\vv_{r'}^\top$ is~$\sum_{j=1}^{N}\sigma_r^{2(N-j)/N}\sigma_{r'}^{2(j-1)/N}$, where~$\sigma_r$ and~$\sigma_{r'}$ are the singular values of~$W_e$ corresponding to~$\uu_r$ and~$\vv_{r'}$ (respectively).
When~$N\geq2$, an increase in~$\sigma_r$ or~$\sigma_{r'}$ leads to an increase in the eigenvalue corresponding to the eigendirection~$\uu_r\vv_{r'}^\top$.
Qualitatively, this implies that the preconditioning favors directions that correspond to singular vectors whose presence in~$W_e$ is stronger.
We conclude that the effect of overparameterization, \ie~of replacing a classic linear model (depth-$1$ network) by a depth-$N$ linear network, boils down to modifying gradient descent by promoting movement along directions that fall in line with the current location in parameter space.
A-priori, such a preference may seem peculiar~--~why should an optimization algorithm be sensitive to its location in parameter space?
Indeed, we generally expect sensible algorithms to be translation invariant, \ie~be oblivious to parameter value.
However, if one takes into account the common practice in deep learning of initializing weights near zero, the location in parameter space can also be regarded as the overall movement made by the algorithm.
We thus interpret our findings as indicating that overparameterization promotes movement along directions already taken by the optimization, and therefore can be seen as a form of acceleration.
This intuitive interpretation will become more concrete in the subsection that follows.

A final point to make, is that the end-to-end update rule (Equation~\ref{eq:We_gd} or~\ref{eq:We_gd_vec}), which obviously depends on~$N$~--~number of layers in the deep linear network, does \emph{not} depend on the hidden widths $n_1\ldots{n}_{N-1}$ (see Section~\ref{sec:lnn}).
This implies that from an optimization perspective, overparameterizing using wide or narrow networks has the same effect~--~it is only the depth that matters.
Consequently, the acceleration of overparameterization can be attained at a minimal computational price, as we demonstrate empirically in Section~\ref{sec:exp}.

  % SINGLE OUTPUT CASE
\subsection{Single Output Case} \label{sec:dynamics:single}

To facilitate a straightforward presentation of our findings, we hereinafter focus on the special case where the optimized models have a single output, \ie~where~$k=1$.
This corresponds, for example, to a binary (two-class) classification problem, or to the prediction of a numeric scalar property (regression).
It admits a particularly simple form for the end-to-end update rule of Equation~\ref{eq:We_gd}:
\begin{claim}
\label{claim:We_gd_single}
Assume~$k=1$, \ie~$W_e\in\R^{1,d}$.
Then, the end-to-end update rule in Equation~\ref{eq:We_gd} can be written as follows:
\vspace{-5mm}
\bea
W_e^{(t+1)}\mapsfrom(1-\eta\lambda{N})\cdot{W}_e^{(t)}
\qquad\qquad\qquad\qquad\qquad
\label{eq:We_gd_single}\\
-\eta\|W_e^{(t)}\|_{2}^{2-\frac{2}{N}}\cdot\left(\tfrac{dL^1}{dW}(W_e^{(t)})+\right.
\qquad\qquad
\nonumber\\[-1mm]
\left.(N-1)\cdot{Pr}_{W_e^{(t)}}\big\{\tfrac{dL^1}{dW}(W_e^{(t)})\big\}\right)
\nonumber
\eea
\vspace{-5mm}\\
where~$\norm{\cdot}_{2}^{2-\frac{2}{N}}$ stands for Euclidean norm raised to the power of~$2-\frac{2}{N}$, and~$Pr_W\{\cdot\}$, $W\in\R^{1,d}$, is defined to be the projection operator onto the direction of~$W$:
\vspace{-1mm}
\bea
&&Pr_W:\R^{1,d}\to\R^{1,d} 
\label{eq:proj}\\
&&Pr_W\{V\}:=
\begin{cases}
\frac{W}{\norm{W}_2}V^\top\cdot\frac{W}{\norm{W}_2} & ,~W\neq0 \\
\qquad\qquad0 & ,~W=0
\end{cases}
\nonumber
\eea
\end{claim}
\vspace{-5mm}
\begin{proof}
The result follows from the definition of a fractional power operator over matrices~--~see Appendix~\ref{app:proofs:We_gd_single}.
\end{proof}

Claim~\ref{claim:We_gd_single} implies that in the single output case, the effect of overparameterization (replacing classic linear model by depth-$N$ linear network) on gradient descent is twofold:
first, it leads to an \emph{adaptive learning rate} schedule, by introducing the multiplicative factor~$\norm{W_e}_{2}^{2-2/N}$;
and second, it amplifies (by~$N$) the projection of the gradient on the direction of~$W_e$.
Recall that we view~$W_e$ not only as the optimized parameter, but also as the overall movement made in optimization (initialization is assumed to be near zero).
Accordingly, the adaptive learning rate schedule can be seen as gaining confidence (increasing step sizes) when optimization moves farther away from initialization, and the gradient projection amplification can be thought of as a certain type of \emph{momentum} that favors movement along the azimuth taken so far.
These effects bear potential to accelerate convergence, as we illustrate qualitatively in Section~\ref{sec:acceleration}, and demonstrate empirically in Section~\ref{sec:exp}.

% OVERPARAMETRIZATION EFFECTS CANNOT BE ATTAINED VIA REGULARIZATION
\section{Overparametrization Effects Cannot Be Attained via Regularization} \label{sec:impossible}

Adding a regularizer to the objective is a standard approach for improving optimization (though lately the term regularization is typically associated with generalization).
For example, AdaGrad was originally invented to compete with the best regularizer from a particular family. 
The next theorem shows (for single output case) that the effects of overparameterization cannot be attained by adding a regularization term to the original training loss, or via any similar modification.
This is not obvious a-priori, as unlike many acceleration methods that explicitly maintain memory of past gradients, updates under overparametrization are by definition the gradients of \emph{something}. 
The assumptions in the theorem are minimal and also necessary, as one must rule-out the trivial counter-example of a constant training loss.
 
\begin{theorem}
\label{theorem:impossible}
Assume~$\frac{dL^1}{dW}$ does not vanish at~$W=0$, and is continuous on some neighborhood around this point.
For a given~$N\in\N$, $N>2$,\note{
For the result to hold with~$N=2$, additional assumptions on~$L^1(\cdot)$ are required; 
otherwise any non-zero linear function~$L^1(W)=WU^\top$ serves as a counter-example~--~it leads to a vector field~$F(\cdot)$ that is the gradient of~$W\mapsto\norm{W}_{2}\cdot{W}U^\top$.
} define:
\vspace{-2mm}
\bea
&F(W):=&
\label{eq:F}\\[-1mm]
&\norm{W}_2^{2{-}\frac{2}{N}}\cdot\left(\tfrac{dL^1}{dW}(W)+(N{-}1)\cdot{Pr}_W\big\{\tfrac{dL^1}{dW}(W)\big\}\right)&
\nonumber
\eea
\vspace{-5mm}\\
where~$Pr_W\{\cdot\}$ is the projection given in Equation~\ref{eq:proj}.
Then, there exists no function (of~$W$) whose gradient field is~$F(\cdot)$.
\end{theorem}
\vspace{-2mm}
\begin{proof} (sketch~--~full details in Appendix~\ref{app:proofs:impossible})
The proof uses elementary differential geometry~\cite{buck2003advanced}: curves, arc length and the fundamental theorem for line integrals, which states that the integral of~$\nabla{g}$ for any differentiable function~$g$ amounts to~$0$ along every closed curve. 

Overparametrization changes gradient descent's behavior: instead of following the original gradient~$\frac{dL^1}{dW}$, it follows some other direction~$F(\cdot)$ (see Equations~\ref{eq:We_gd_single} and~\ref{eq:F}) that is a \emph{function} of the original gradient as well as the current point~$W$. 
We think of this change as a transformation that maps one \emph{vector field} $\phi(\cdot)$ to another~--~$F_\phi(\cdot)$:
\vspace{-1mm}
\beas
F_\phi(W)= 
\quad\qquad\qquad\qquad\qquad\\[-0.5mm]
\begin{cases}
\hspace{-0.5mm}\norm{W}^{2{-}\frac{2}{N}}\hspace{-1mm}\left(\phi(W){+}(N{-}1)\hspace{-0.5mm}\inprod{\phi(W)}{\frac{W}{\norm{W}}}\hspace{-0.5mm}\frac{W}{\norm{W}}\right) 
& \hspace{-2.5mm},W{\neq}0 \\[-0.5mm]
\qquad\qquad\qquad\qquad\qquad0 
& \hspace{-2.5mm},W{=}0
\end{cases}
\eeas
\vspace{-3mm}\\
Notice that for $\phi=\frac{dL^1}{dW}$, we get exactly the vector field $F(\cdot)$ defined in theorem statement.

We note simple properties of the mapping $\phi\mapsto{F}_\phi$. 
First, it is linear, since for any vector fields $\phi_1,\phi_2$ and scalar~$c$: $F_{\phi_1{+}\phi_2}{=}F_{\phi_1}{+}F_{\phi_2}$ and $F_{c{\cdot}\phi_1}{=}c{\cdot}{F}_{\phi_1}$.
Second, because of the linearity of line integrals, for any curve~$\Gamma$, the functional $\phi\mapsto\int_{\Gamma}F_\phi$, a mapping of vector fields to scalars, is linear.

We show that~$F(\cdot)$ contradicts the fundamental theorem for line integrals.
To do so, we construct a closed curve~$\Gamma{=}\Gamma_{r,R}$ for which the linear functional $\phi\mapsto\oint_\Gamma{F}_\phi$ does not vanish at $\phi{=}\frac{dL^1}{dW}$. 
Let $\e:=\frac{dL^1}{dW}(W{=}0)/\|\frac{dL^1}{dW}(W{=}0)\|$, which is well-defined since by assumption $\frac{dL^1}{dW}(W{=}0){\neq0}$. 
For $r<R$ we define (see Figure~\ref{fig:curve}):
\vspace{-1mm}
$$\Gamma_{r,R}:=\Gamma_{r,R}^1~\to~\Gamma_{r,R}^2~\to~\Gamma_{r,R}^3~\to~\Gamma_{r,R}^4$$
\vspace{-7mm}\\
where:
\begin{itemize}
\vspace{-3mm}
\item $\Gamma_{r,R}^1$ is the line segment from~$-R\cdot\e$ to~$-r\cdot\e$. 
\vspace{-2.2mm}
\item $\Gamma_{r,R}^2$ is a spherical curve from~$-r\cdot\e$ to~$r\cdot\e$.
\vspace{-2.2mm}
\item $\Gamma_{r,R}^3$ is the line segment from~$r\cdot\e$ to~$R\cdot\e$.
\vspace{-2.2mm}
\item $\Gamma_{r,R}^4$ is a spherical curve from~$R\cdot\e$ to~$-R\cdot\e$.
\vspace{-2.2mm}
\end{itemize}
With the definition of~$\Gamma_{r,R}$ in place, we decompose $\frac{dL^1}{dW}$ into a constant vector field $\kappa\,{\equiv}\,\frac{dL^1}{dW}(W{=}0)$ plus a residual~$\xi$.
We explicitly compute the line integrals along $\Gamma_{r,R}^1\ldots\Gamma_{r,R}^4$ for~$F_\kappa$, and derive bounds for~$F_\xi$.
This, along with the linearity of the functional $\phi\mapsto\int_{\Gamma}F_\phi$, provides a lower bound on the line integral of~$F(\cdot)$ over~$\Gamma_{r,R}$.
We show the lower bound is positive as $r,R\to0$, thus $F(\cdot)$ indeed contradicts the fundamental theorem for line integrals.
\end{proof}

\begin{figure}
\vspace{-2mm}
\begin{center}
\includegraphics[width=0.6\columnwidth]{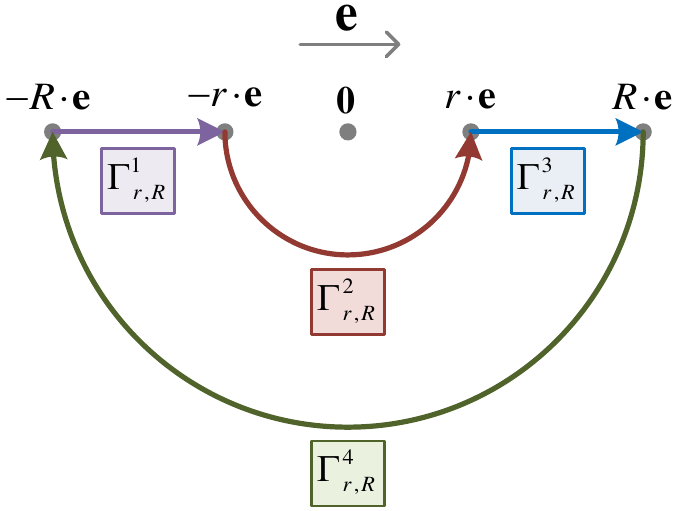}
\end{center}
\vspace{-4mm}
\caption{Curve $\Gamma_{r,R}$ over which line integral is non-zero.}
\label{fig:curve}
\end{figure}

% ILLUSTRATION OF ACCELERATION
\section{Illustration of Acceleration} \label{sec:acceleration}

To this end, we showed that overparameterization (use of depth-$N$ linear network in place of classic linear model) induces on gradient descent a particular preconditioning scheme (Equation~\ref{eq:We_gd} in general and~\ref{eq:We_gd_single} in the single output case), which can be interpreted as introducing some forms of momentum and adaptive learning rate.
We now illustrate qualitatively, on a very simple hypothetical learning problem, the potential of these to accelerate optimization.

Consider the task of linear regression, assigning to vectors in~$\R^2$ labels in~$\R$.
Suppose that our training set consists of two points in $\R^2\times\R$: $([1,0]^\top,y_1)$ and $([0,1]^\top,y_2)$.
Assume also that the loss function of interest is~$\ell_p$, $p\in2\N$: $\ell_p(\hat{y},y)=\frac{1}{p}(\hat{y}-y)^p$.
Denoting the learned parameter by $\w=[w_1,w_2]^\top$, the overall training loss can be written as:\note{
We omit the averaging constant~$\frac{1}{2}$ for conciseness.
}
\vspace{-2mm}
$$L(w_1,w_2)=\tfrac{1}{p}(w_1-y_1)^p+\tfrac{1}{p}(w_2-y_2)^p$$

With fixed learning rate~$\eta>0$ (weight decay omitted for simplicity), gradient descent over~$L(\cdot)$ gives:
\vspace{-1mm}
$$w_i^{(t+1)}\mapsfrom{w}_i^{(t)}-\eta(w_i^{(t)}-y_i)^{p-1}\quad,~i=1,2$$
\vspace{-5mm}\\
Changing variables per $\Delta_i=w_i-y_i$, we have:
\vspace{-1mm}
\be
\Delta_i^{(t+1)}\mapsfrom{\Delta}_i^{(t)}\big(1-\eta(\Delta_i^{(t)})^{p-2}\big)\quad,~i=1,2
\label{eq:illus_gd_delta}
\ee
\vspace{-5mm}\\
Assuming the original weights $w_1$ and~$w_2$ are initialized near zero, $\Delta_1$ and $\Delta_2$ start off at $-y_1$ and~$-y_2$ respectively, and will eventually reach the optimum $\Delta^*_1=\Delta^*_2=0$ if the learning rate is small enough to prevent divergence:
\vspace{-1mm}
$$\eta<\tfrac{2}{y_i^{p-2}}\quad,~i=1,2$$
\vspace{-5mm}\\
Suppose now that the problem is ill-conditioned, in the sense that $y_1{\gg}{y}_2$.
If $p=2$ this has no effect on the bound for~$\eta$.\note{
Optimal learning rate for gradient descent on quadratic objective does not depend on current parameter value~(\cf~\citet{goh2017why}).
}
If $p>2$ the learning rate is determined by~$y_1$, leading~$\Delta_2$ to converge very slowly.
In a sense, $\Delta_2$~will suffer from the fact that there is no ``communication'' between the coordinates (this will actually be the case not just with gradient descent, but with most algorithms typically used in large-scale settings~--~AdaGrad, Adam, \etc).

Now consider the scenario where we optimize~$L(\cdot)$ via overparameterization, \ie~with the update rule in Equation~\ref{eq:We_gd_single} (single output).
In this case the coordinates are coupled, and as~$\Delta_1$ gets small ($w_1$ gets close to $y_1$), the learning rate is effectively scaled by~$y_1^{2-\frac{2}{N}}$ (in addition to a scaling by~$N$ in coordinate~$1$ only), allowing (if $y_1{>}1$) faster convergence of~$\Delta_2$.
We thus have the luxury of temporarily slowing down~$\Delta_2$ to ensure that~$\Delta_1$ does not diverge, with the latter speeding up the former as it reaches safe grounds.
In Appendix~\ref{app:acceleration_bound} we consider a special case and formalize this intuition, deriving a concrete bound for the acceleration~of~overparameterization.

% EXPERIMENTS
\section{Experiments} \label{sec:exp}

Our analysis (Section~\ref{sec:dynamics}) suggests that overparameterization~--~replacement of a classic linear model by a deep linear network, induces on gradient descent a certain preconditioning scheme.
We qualitatively argued (Section~\ref{sec:acceleration}) that in some cases, this preconditioning may accelerate convergence.
In this section we put these claims to the test, through a series of empirical evaluations based on TensorFlow toolbox (\citet{abadi2016tensorflow}).
For conciseness, many of the details behind our implementation are deferred to Appendix~\ref{app:impl}.

We begin by evaluating our analytically-derived preconditioning scheme~--~the end-to-end update rule in Equation~\ref{eq:We_gd}.
Our objective in this experiment is to ensure that our analysis, continuous in nature and based on a particular assumption on weight initialization (Equation~\ref{eq:Wj_agree}), is indeed applicable to practical scenarios.
We focus on the single output case, where the update-rule takes on a particularly simple (and efficiently implementable) form~--~Equation~\ref{eq:We_gd_single}.
The dataset chosen was UCI Machine Learning Repository's ``Gas Sensor Array Drift at Different Concentrations''~\cite{vergara2012chemical,rodriguez2014calibration}.
Specifically, we used the dataset's ``Ethanol'' problem~--~a scalar regression task with~$2565$ examples, each comprising~$128$ features (one of the largest numeric regression tasks in the repository).
As training objectives, we tried both~$\ell_2$ and~$\ell_4$ losses.
Figure~\ref{fig:exp_update_rule} shows convergence (training objective per iteration) of gradient descent optimizing depth-$2$ and depth-$3$ linear networks, against optimization of a single layer model using the respective preconditioning schemes (Equation~\ref{eq:We_gd_single} with~$N=2,3$).
As can be seen, the preconditioning schemes reliably emulate deep network optimization, suggesting that, at least in some cases, our analysis indeed captures practical dynamics.

\begin{figure}
\vspace{-3mm}
\begin{center}
\includegraphics[width=0.49\columnwidth]{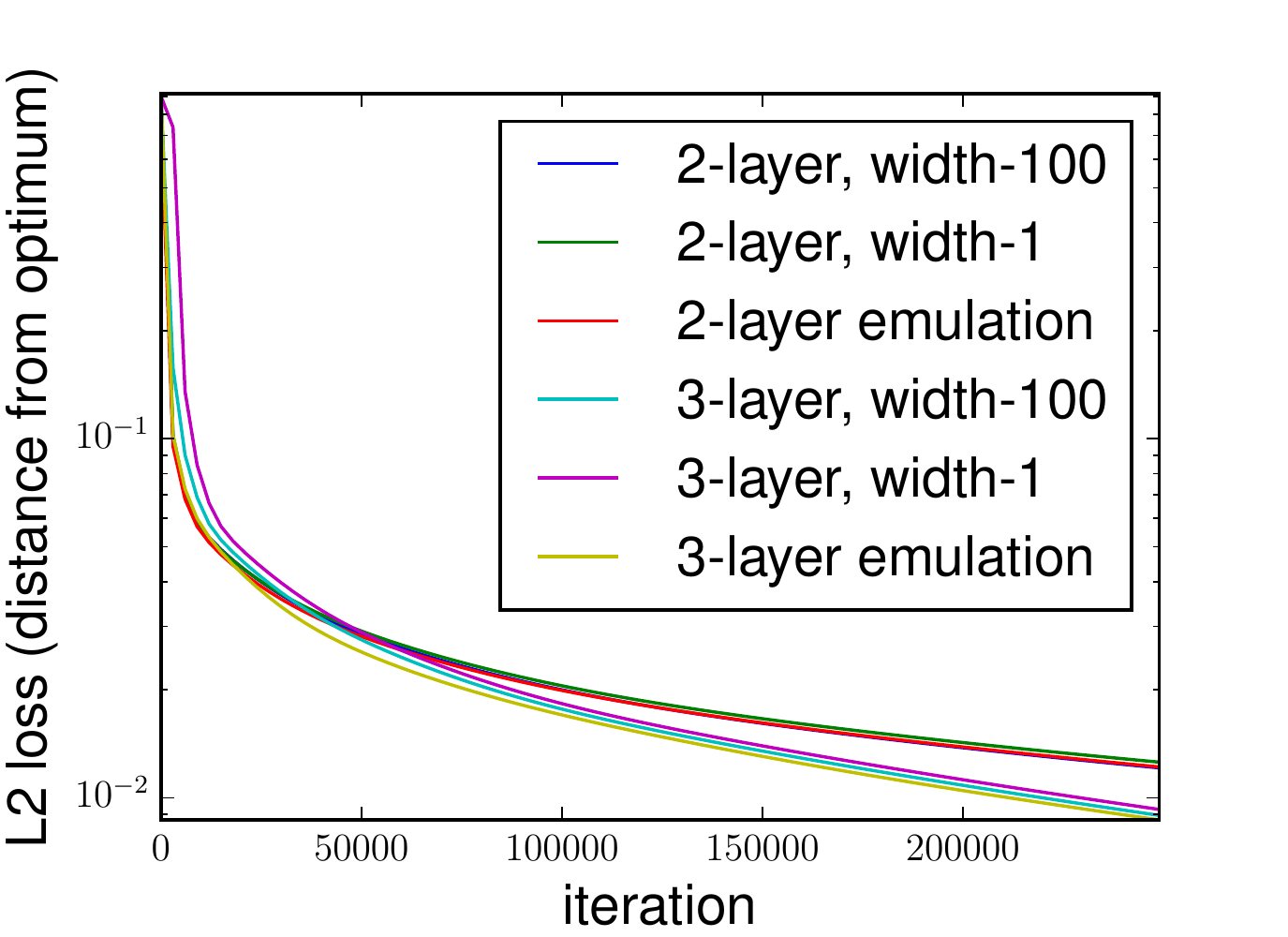}
\includegraphics[width=0.49\columnwidth]{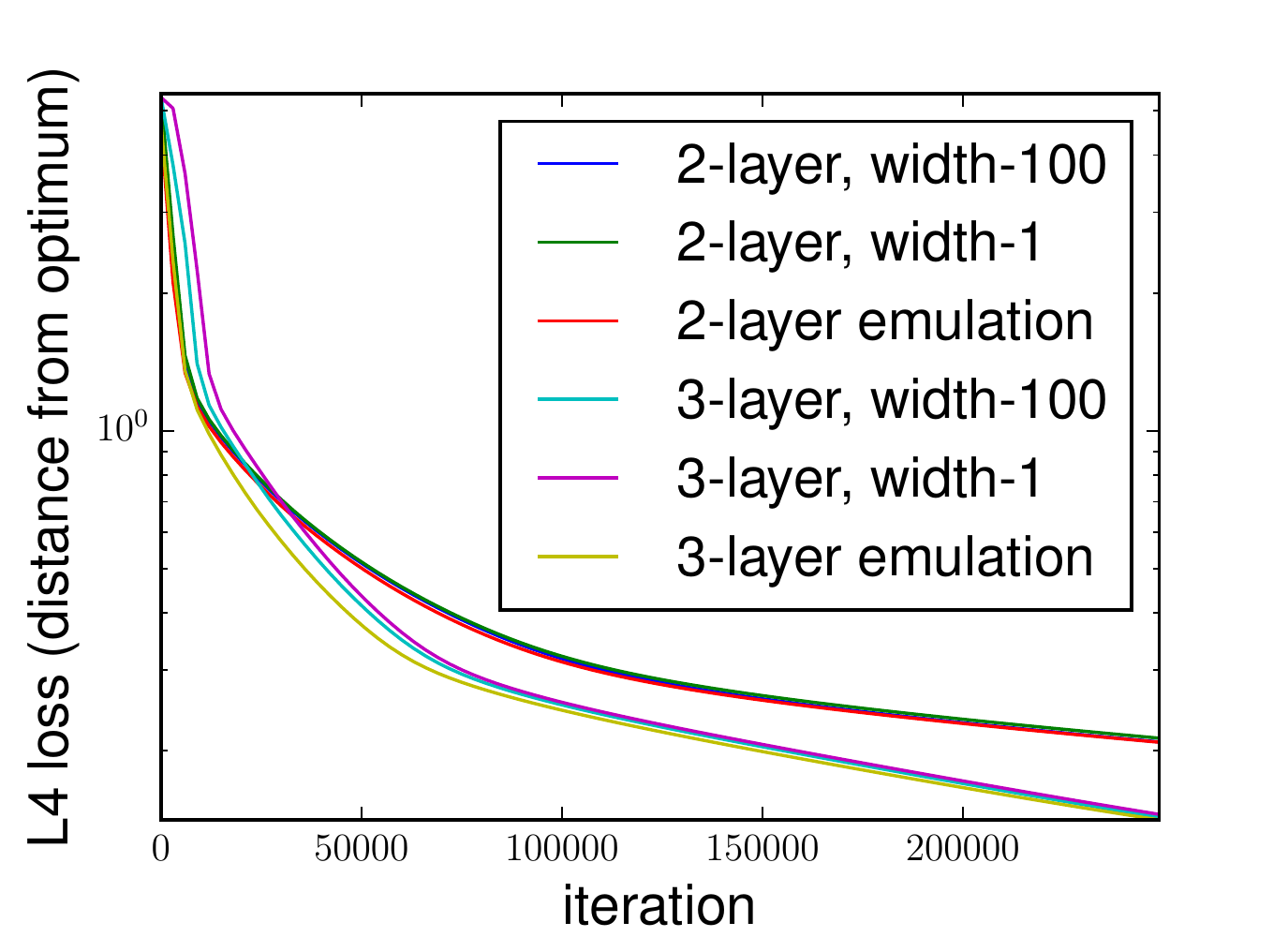}
\end{center}
\vspace{-6mm}
\caption{
(to be viewed in color)~
Gradient descent optimization of deep linear networks (depths~$2,3$) \vs~the analytically-derived equivalent preconditioning schemes (over single layer model; Equation~\ref{eq:We_gd_single}).
Both plots show training objective (left~--~$\ell_2$~loss; right~--~$\ell_4$~loss) per iteration, on a numeric regression dataset from UCI Machine Learning Repository (details in text).
Notice the emulation of preconditioning schemes.
Notice also the negligible effect of network width~--~for a given depth, setting size of hidden layers to~$1$ (scalars) or~$100$ yielded similar convergence (on par with our analysis).
}
\label{fig:exp_update_rule}
\vspace{-3mm}
\end{figure}

Alongside the validity of the end-to-end update rule, Figure~\ref{fig:exp_update_rule} also demonstrates the negligible effect of network width on convergence, in accordance with our analysis (see Section~\ref{sec:dynamics}).
Specifically, it shows that in the evaluated setting, hidden layers of size~$1$ (scalars) suffice in order for the essence of overparameterization to fully emerge.
Unless otherwise indicated, all results reported hereinafter are based on this configuration, \ie~on scalar hidden layers.
The computational toll associated with overparameterization will thus be virtually non-existent.

As a final observation on Figure~\ref{fig:exp_update_rule}, notice that it exhibits faster convergence with a deeper network.
This however does not serve as evidence in favor of acceleration by depth, as we did not set learning rates optimally per model (simply used the common choice of~$10^{-3}$).
To conduct a fair comparison between the networks, and more importantly, between them and a classic single layer model, multiple learning rates were tried, and the one giving fastest convergence was taken on a per-model basis.
Figure~\ref{fig:exp_main} shows the results of this experiment.
As can be seen, convergence of deeper networks is (slightly) slower in the case of~$\ell_2$ loss.
This falls in line with the findings of~\citet{saxe2013exact}.
In stark contrast, and on par with our qualitative analysis in Section~\ref{sec:acceleration}, is the fact that with~$\ell_4$ loss adding depth significantly accelerated convergence.
To the best of our knowledge, this provides first empirical evidence to the fact that depth, even without any gain in expressiveness, and despite introducing non-convexity to a formerly convex problem, can lead to favorable optimization.

\begin{figure}
\vspace{-3mm}
\begin{center}
\includegraphics[width=0.49\columnwidth]{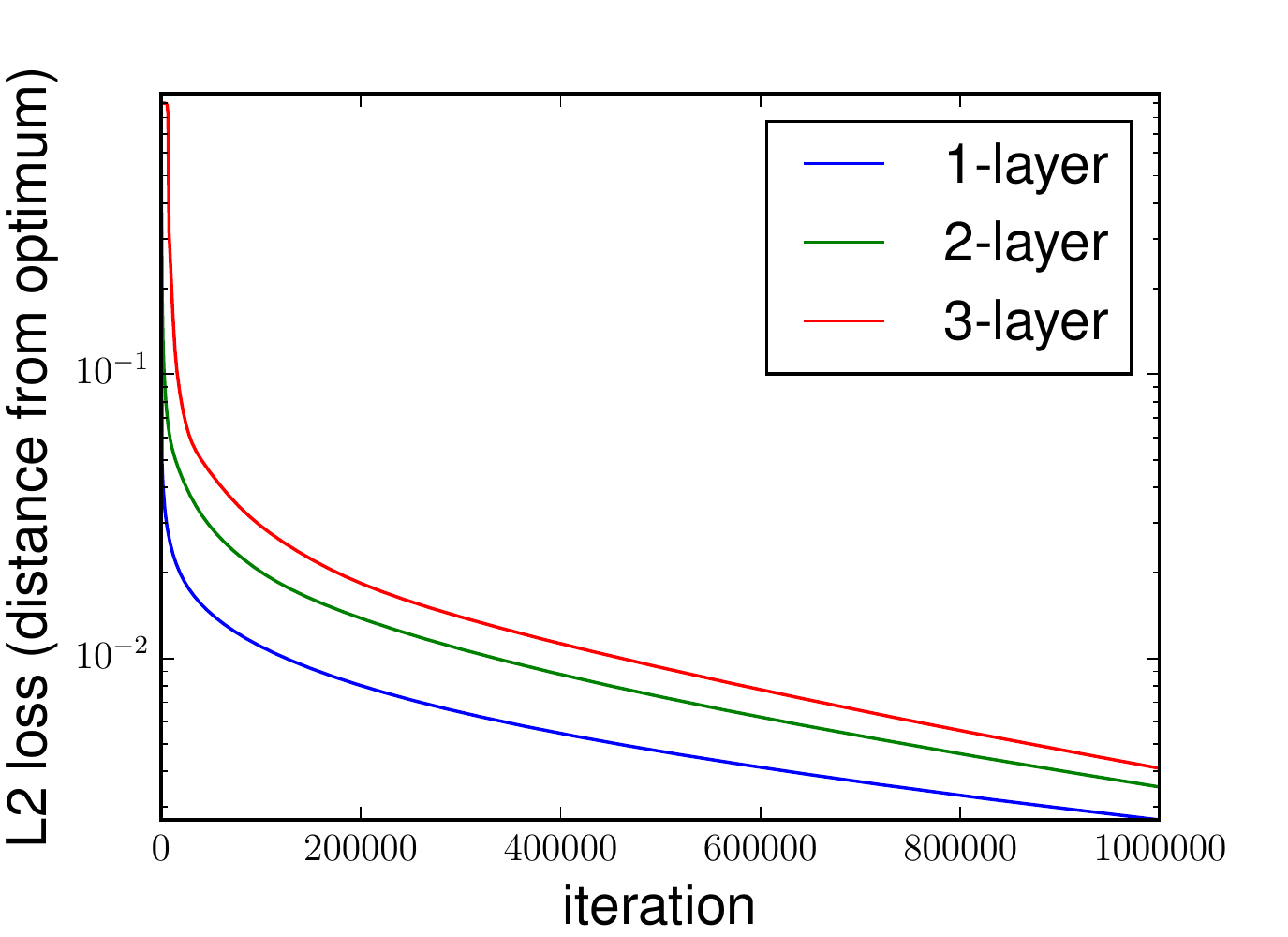}
\includegraphics[width=0.49\columnwidth]{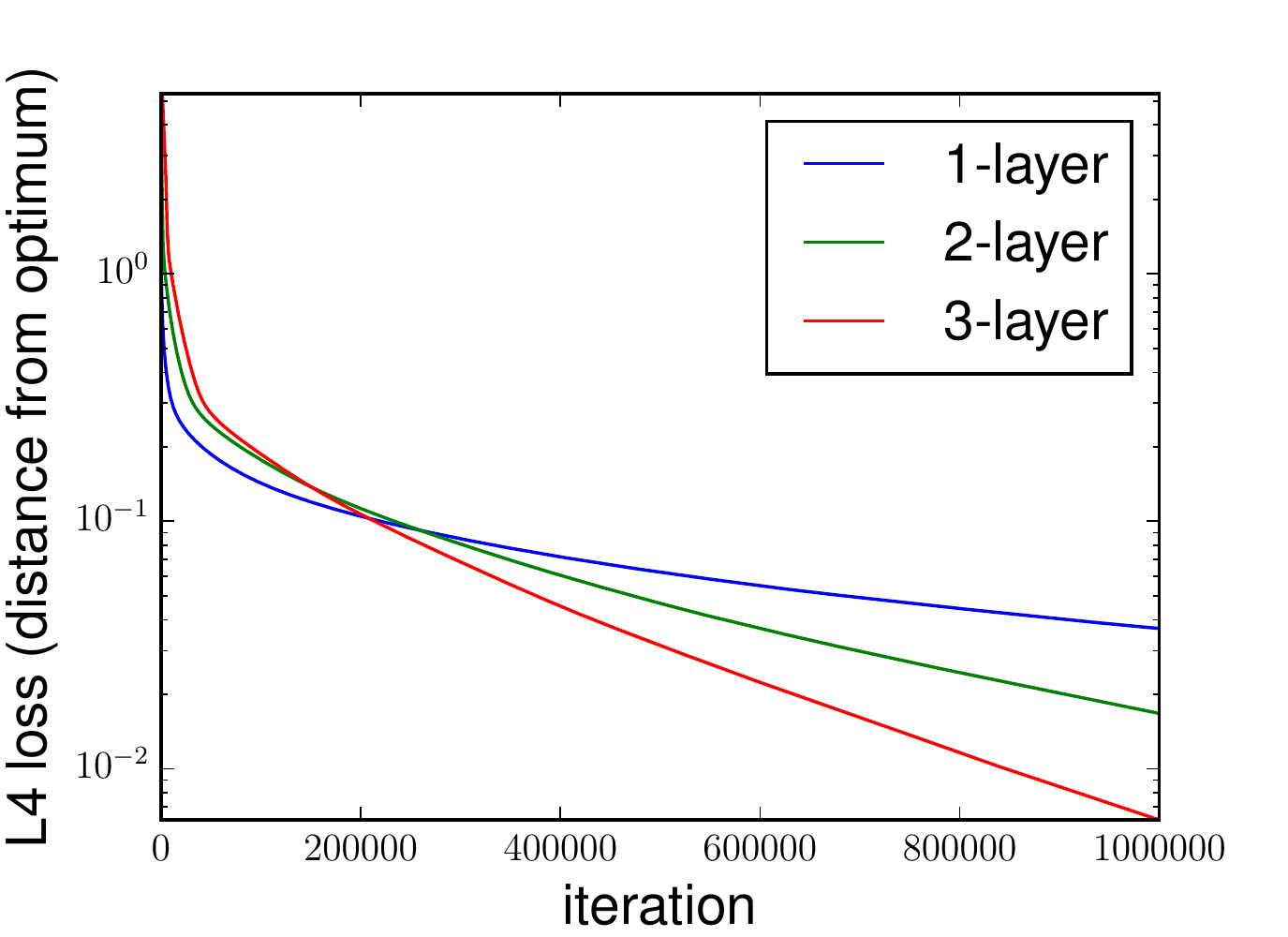}
\end{center}
\vspace{-6mm}
\caption{
(to be viewed in color)~
Gradient descent optimization of single layer model \vs~linear networks of depth~$2$ and~$3$.
Setup is identical to that of Figure~\ref{fig:exp_update_rule}, except that here learning rates were chosen via grid search, individually per model (see Appendix~\ref{app:impl}).
Notice that with $\ell_2$ loss, depth (slightly) hinders optimization, whereas with~$\ell_4$ loss it leads to significant acceleration (on par with our qualitative analysis in Section~\ref{sec:acceleration}).
}
\label{fig:exp_main}
\vspace{-3mm}
\end{figure}

In light of the speedup observed with~$\ell_4$ loss, it is natural to ask how the implicit acceleration of depth compares against explicit methods for acceleration and adaptive learning.
Figure~\ref{fig:exp_ada}-left shows convergence of a depth-$3$ network (optimized with gradient descent) against that of a single layer model optimized with AdaGrad~\cite{duchi2011adaptive} and AdaDelta~\cite{zeiler2012adadelta}.
The displayed curves correspond to optimal learning rates, chosen individually via grid search.
Quite surprisingly, we find that in this specific setting, overparameterizing, thereby turning a convex problem non-convex, is a more effective optimization strategy than carefully designed algorithms tailored for convex problems.
We note that this was not observed with all algorithms~--~for example Adam~\cite{kingma2014adam} was considerably faster than overparameterization.
However, when introducing overparameterization simultaneously with Adam (a setting we did not theoretically analyze), further acceleration is attained~--~see Figure~\ref{fig:exp_ada}-right.
This suggests that at least in some cases, not only plain gradient descent benefits from depth, but also more elaborate algorithms commonly employed in state of the art applications.

\begin{figure}
\vspace{-3mm}
\begin{center}
\includegraphics[width=0.49\columnwidth]{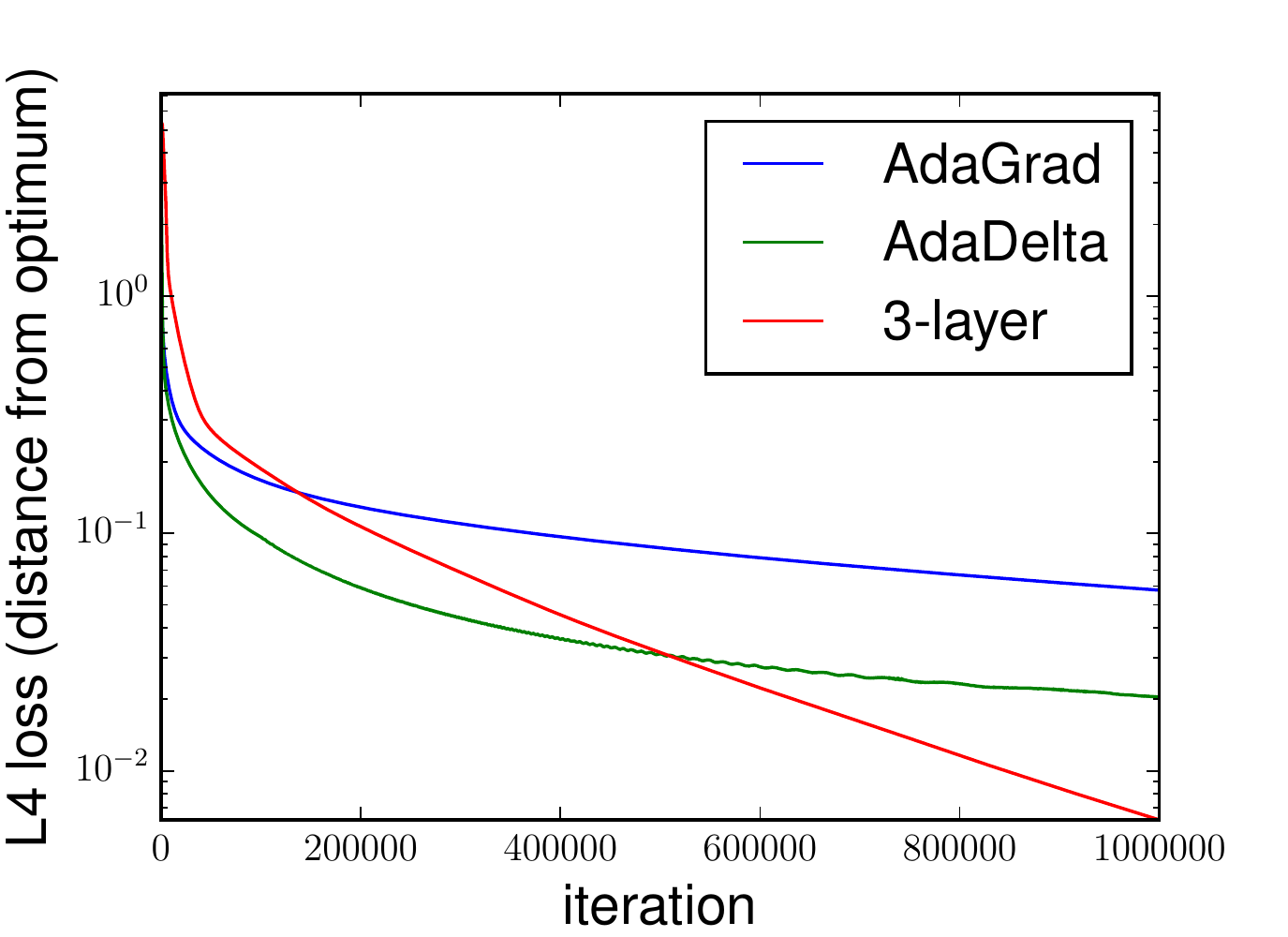}
\includegraphics[width=0.49\columnwidth]{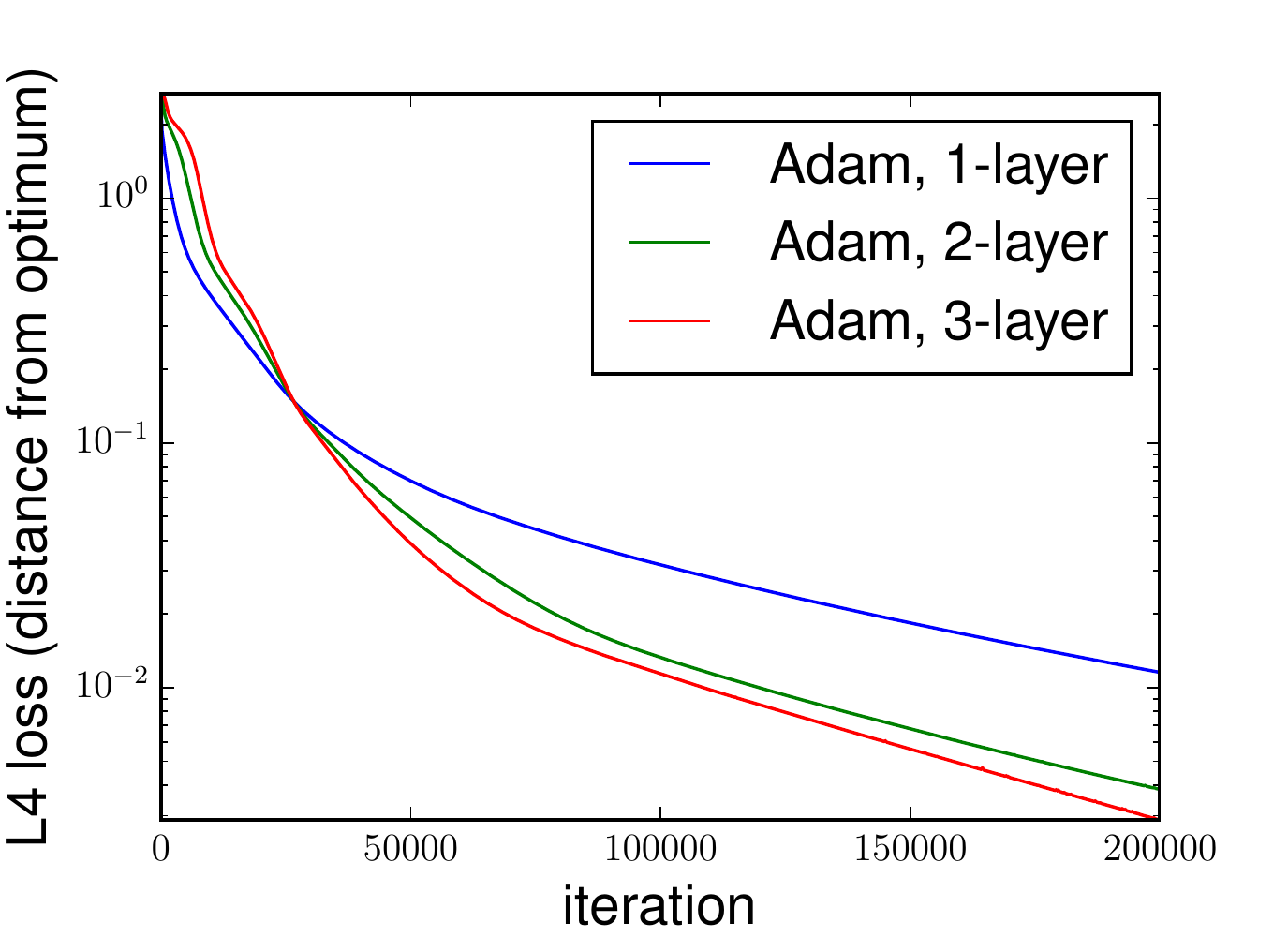}
\end{center}
\vspace{-6mm}
\caption{
(to be viewed in color)~
\textbf{Left:}
Gradient descent optimization of depth-$3$ linear network \vs~AdaGrad and AdaDelta over single layer model.
Setup is identical to that of Figure~\ref{fig:exp_main}-right.
Notice that the implicit acceleration of overparameterization outperforms both AdaGrad and AdaDelta (former is actually slower than plain gradient descent).
\textbf{Right:}
Adam optimization of single layer model \vs~Adam over linear networks of depth~$2$ and~$3$.
Same setup, but with learning rates set per Adam's default in TensorFlow.
Notice that depth improves speed, suggesting that the acceleration of overparameterization may be somewhat orthogonal to explicit acceleration methods.
}
\label{fig:exp_ada}
\vspace{-3mm}
\end{figure}

An immediate question arises at this point.
If depth indeed accelerates convergence, why not add as many layers as one can computationally afford?
The reason, which is actually apparent in our analysis, is the so-called \emph{vanishing gradient problem}.
When training a very deep network (large~$N$), while initializing weights to be small, the end-to-end matrix~$W_e$ (Equation~\ref{eq:We}) is extremely close to zero, severely attenuating gradients in the preconditioning scheme (Equation~\ref{eq:We_gd}).
A possible approach for alleviating this issue is to initialize weights to be larger, yet small enough such that the end-to-end matrix does not ``explode''.
The choice of identity (or near identity) initialization leads to what is known as \emph{linear residual networks}~\cite{hardt2016identity}, akin to the successful residual networks architecture~\cite{he2015deep} commonly employed in deep learning.
Notice that identity initialization satisfies the condition in Equation~\ref{eq:Wj_agree}, rendering the end-to-end update rule (Equation~\ref{eq:We_gd}) applicable.
Figure~\ref{fig:exp_resnet_cnn}-left shows convergence, under gradient descent, of a single layer model against deeper networks than those evaluated before~--~depths~$4$ and~$8$.
As can be seen, with standard, near-zero initialization, the depth-$4$ network starts making visible progress only after about~$65K$ iterations, whereas the depth-$8$ network seems stuck even after~$100K$ iterations.
In contrast, under identity initialization, both networks immediately make progress, and again depth serves as an implicit~accelerator.

As a final sanity test, we evaluate the effect of overparameterization on optimization in a non-idealized (yet simple) deep learning setting.
Specifically, we experiment with the convolutional network tutorial for MNIST built into TensorFlow,\note{
\url{https://github.com/tensorflow/models/tree/master/tutorials/image/mnist}
}
which includes convolution, pooling and dense layers, ReLU non-linearities, stochastic gradient descent with momentum, and dropout~\cite{srivastava2014dropout}.
We introduced overparameterization by simply placing two matrices in succession instead of the matrix in each dense layer.
Here, as opposed to previous experiments, widths of the newly formed hidden layers were not set to~$1$, but rather to the minimal values that do not deteriorate expressiveness (see Appendix~\ref{app:impl}).
Overall, with an addition of roughly~$15\%$ in number of parameters, optimization has accelerated considerably~--~see Figure~\ref{fig:exp_resnet_cnn}-right.
The displayed results were obtained with the hyperparameter settings hardcoded into the tutorial.
We have tried alternative settings (varying learning rates and standard deviations of initializations~--~see Appendix~\ref{app:impl}), and in all cases observed an outcome similar to that in Figure~\ref{fig:exp_resnet_cnn}-right~--~overparameterization led to significant speedup.
Nevertheless, as reported above for linear networks, it is likely that for non-linear networks the effect of depth on optimization is mixed~--~some settings accelerate by it, while others do not.
Comprehensive characterization of the cases in which depth accelerates optimization warrants much further study.
We hope our work will spur interest in this avenue of research.

\begin{figure}
\vspace{-3mm}
\begin{center}
\includegraphics[width=0.49\columnwidth]{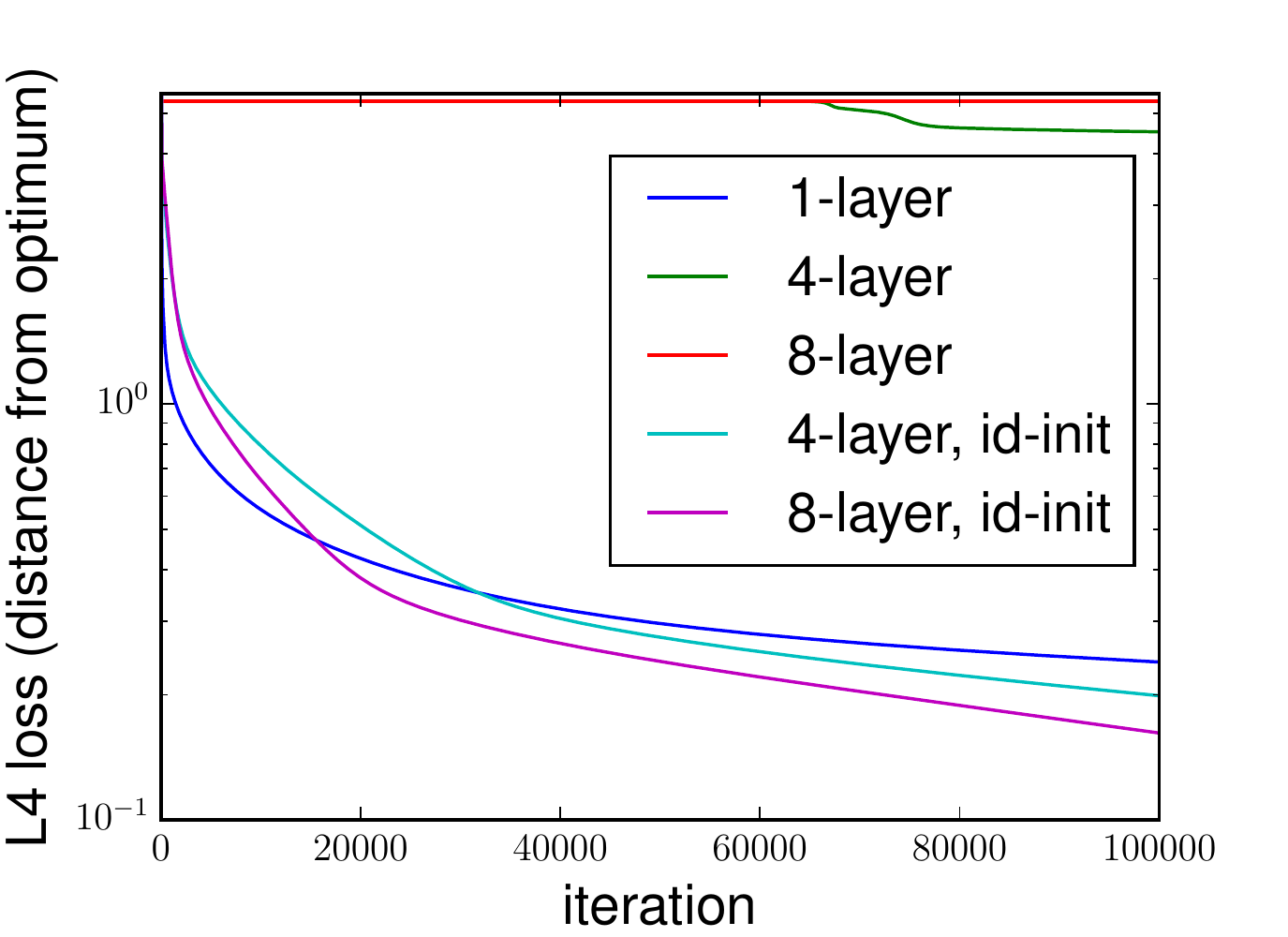}
\includegraphics[width=0.49\columnwidth]{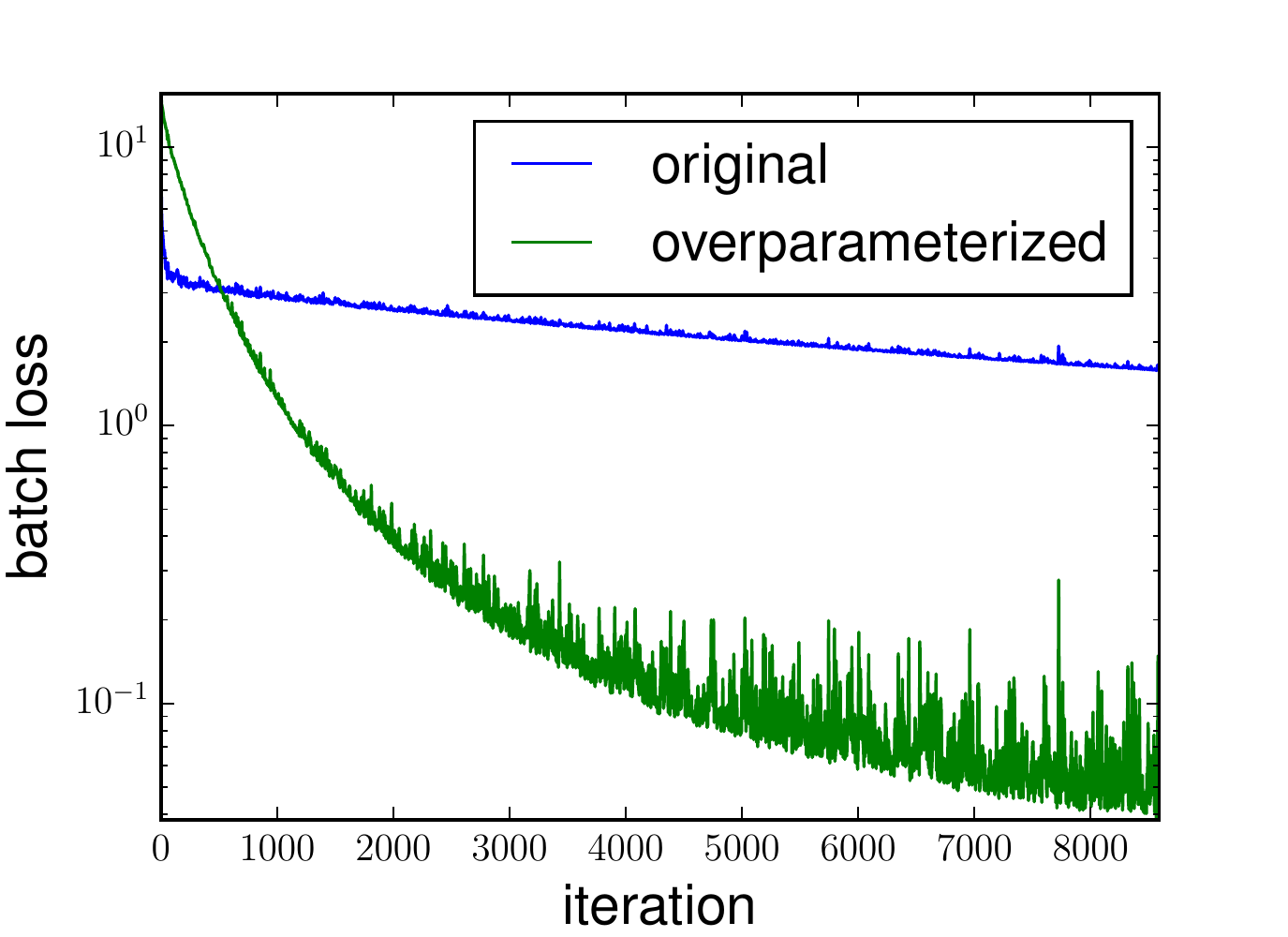}
\end{center}
\vspace{-6mm}
\caption{
(to be viewed in color)~
\textbf{Left:}
Gradient descent optimization of single layer model \vs~linear networks deeper than before (depths~$4,8$).
For deep networks, both near-zero and near-identity initializations were evaluated.
Setup identical to that of Figure~\ref{fig:exp_main}-right.
Notice that deep networks suffer from vanishing gradients under near-zero initialization, while near-identity (``residual'') initialization eliminates the problem.
\textbf{Right:}
Stochastic gradient descent optimization in TensorFlow's convolutional network tutorial for MNIST.
Plot shows batch loss per iteration, in original setting \vs~overparameterized one (depth-$2$ linear networks in place of dense layers).
}
\label{fig:exp_resnet_cnn}
\vspace{-3mm}
\end{figure}

% CONCLUSION
\section{Conclusion} \label{sec:conc}
\vspace{-1mm}

Through theory and experiments, we demonstrated that overparameterizing a neural network by increasing its depth can accelerate optimization, even on very simple problems.

Our analysis of linear neural networks, the subject of various recent studies, yielded a new result: for these models, overparameterization by depth can be understood as a preconditioning scheme with a closed form description (Theorem~\ref{theorem:We_gf} and the claims thereafter).
The preconditioning may be interpreted as a combination between certain forms of adaptive learning rate and momentum.
Given that it depends on network depth but not on width, acceleration by overparameterization can be attained at a minimal computational price, as we demonstrate empirically in Section~\ref{sec:exp}.

Clearly, complete theoretical analysis for non-linear networks will be challenging.
Empirically however, we showed that the trivial idea of replacing an internal weight matrix by a product of two can significantly accelerate optimization, with absolutely no effect on expressiveness (Figure~\ref{fig:exp_resnet_cnn}-right).

The fact that gradient descent over classic convex problems such as linear regression with $\ell_p$~loss, $p>2$, can accelerate from transitioning to a non-convex overparameterized objective, does not coincide with conventional wisdom, and provides food for thought.
Can this effect be rigorously quantified, similarly to analyses of explicit acceleration methods such as momentum or adaptive regularization (AdaGrad)?

% ACKNOWLEDGMENTS
\newcommand{\acknowledgments}
{Sanjeev Arora's work is supported by NSF, ONR, Simons Foundation, Schmidt Foundation, Mozilla Research, Amazon Research, DARPA and SRC.
Elad Hazan's work is supported by NSF grant 1523815 and Google Brain.
Nadav Cohen is a member of the Zuckerman Israeli Postdoctoral Scholars Program, and is supported by Eric and Wendy~Schmidt.}
\ifdefined\COLT
	\acks{\acknowledgments}
\else
	\ifdefined\CAMREADY
		\vspace{-2mm}
		\section*{Acknowledgments}
		\vspace{-2mm}
		\acknowledgments
	\fi
\fi

% REFERENCES
\section*{References}
{\small
\ifdefined\ICML
	\bibliographystyle{icml2018}
\else
	\bibliographystyle{plainnat}
\fi
\bibliography{refs.bib}
}

% APPENDIXES
\clearpage
\appendix

% ENDNOTES
\ifdefined\NOTESAPP
	\theendnotes
\fi

% DEFERRED PROOFS
\section{Deferred Proofs} \label{app:proofs}

  % END-TO-END DIFFERENTIAL EQUATION
\subsection{Proof of Theorem~\ref{theorem:We_gf}} \label{app:proofs:We_gf}

Before delving into the proof, we introduce notation that will admit a more compact presentation of formulae.
For~$1\leq{a}\leq{b}\leq{N}$, we denote:
\beas
&&\prod\nolimits^{j=b}_{a}W_j~~:=W_{b}W_{b-1}\cdots{W}_a \\
&&\prod\nolimits_{j=a}^{b}W_j^\top:=W^\top_{a}W^\top_{a+1}\cdots{W}^\top_b
\eeas
where~$W_1\ldots{W}_N$ are the weight matrices of the depth-$N$ linear network (Equation~\ref{eq:lnn}).
If~$a>b$, then by definition both $\prod^{j=b}_{a}W_j$ and $\prod_{j=a}^{b}W_j^\top$ are identity matrices, with size depending on context, \ie~on the dimensions of matrices they are multiplied against.
Given any square matrices (possibly scalars) $A_1,A_2,\ldots,A_m$, we denote by $diag(A_1\ldots{A}_m)$ a block-diagonal matrix holding them on its diagonal:
$$diag(A_1\ldots{A}_m) = 
\begin{bmatrix} 
A_1 & 0 & 0 & 0 \\[-0.25em]
0 & \ddots & 0 & 0 \\[0.25em]
0 & 0 & A_m & 0 \\[0.25em]
0 & 0 & 0 & 0 
\end{bmatrix}$$
As illustrated above, $diag(A_1\ldots{A}_m)$ may hold additional, zero-valued rows and columns beyond~$A_1\ldots{A}_m$.
Conversely, it may also trim (omit) rows and columns, from its bottom and right ends respectively, so long as only zeros are being removed.
The exact shape of $diag(A_1\ldots{A}_m)$ is again determined by context, and so if~$B$ and~$C$ are matrices, the expression $B\cdot{diag}(A_1\ldots{A}_m)\cdot{C}$ infers a number of rows equal to the number of columns in~$B$, and a number of columns equal to the number of rows in~$C$.

\medskip

Turning to the actual proof, we disregard the trivial case $N=1$, and begin by noticing that Equation~\ref{eq:lnn_loss_oprm}, along with the definition of~$W_e$ (Equation~\ref{eq:We}), imply that for every $j=1\ldots{N}$:
$$\frac{\partial{L^N}}{\partial{W_j}}(W_1,{\ldots},W_N)=\prod_{i=j+1}^{N}W_i^\top\cdot\frac{dL^1}{dW}(W_e)\cdot\prod_{i=1}^{j-1}W_i^\top$$
Plugging this into the differential equations of gradient descent (Equation~\ref{eq:Wj_gf}), we get:
\bea
\dot{W}_j(t)=-\eta\lambda{W}_j(t)
~~\quad\qquad\qquad\qquad\qquad\qquad\qquad
\label{eq:Wj_gf_L1}\\
-\eta\prod_{i=j+1}^{N}W_i^\top(t)\cdot\frac{dL^1}{dW}(W_e(t))\cdot\prod_{i=1}^{j-1}W_i^\top(t)
\nonumber\\[1mm]
,~j=1{\ldots}N
\quad
\nonumber
\eea
For $j=1\ldots{N}{-}1$, multiply the~$j$'th equation by~$W_j^\top(t)$ from the right, and the~$j{+}1$'th equation by~$W_{j+1}^\top(t)$ from the left.
This yields:
\beas
&W_{j+1}^\top(t)\dot{W}_{j+1}(t)+\eta\lambda\cdot{W}_{j+1}^\top(t)W_{j+1}(t)=& \\
&\dot{W}_j(t)W_j^\top(t)+\eta\lambda\cdot{W}_j(t)W_j^\top(t)&
\eeas
Taking the transpose of these equations and adding to themselves, we obtain, for every~$j=1\ldots{N}{-}1$:
\bea
&W_{j+1}^\top(t)\dot{W}_{j+1}(t)+\dot{W}_{j+1}^\top(t)W_{j+1}(t)+& 
\nonumber\\
&2\eta\lambda\cdot{W}_{j+1}^\top(t)W_{j+1}(t)=& 
\nonumber\\
&\dot{W}_j(t)W_j^\top(t)+W_j(t)\dot{W}_j^\top(t)+& 
\nonumber\\
&2\eta\lambda\cdot{W}_j(t)W_j^\top(t)&
\label{eq:Wj_Wjt_de}
\eea
Denote for $j=1\ldots{N}$:
$$C_j(t):=W_j(t)W_j^\top(t)~~~,~~~C'_j(t):=W_j^\top(t)W_j(t)$$
Equation~\ref{eq:Wj_Wjt_de} can now be written as:
\beas
\dot{C}'_{j+1}(t)+2\eta\lambda\cdot{C}'_{j+1}(t)=\dot{C}_j(t)+2\eta\lambda\cdot{C}_j(t) \\[1mm]
,~j=1{\ldots}N-1
\eeas
Turning to Lemma~\ref{lemma:coincide} below, while recalling our assumption for time~$t_0$ (Equation~\ref{eq:Wj_agree}):
$$C'_{j+1}(t_0)=C_j(t_0)\quad,~j=1{\ldots}N-1$$
we conclude that, throughout the entire time-line:
$$C'_{j+1}(t)=C_j(t)\quad,~j=1{\ldots}N-1$$
Recollecting the definitions of~$C_j(t),C'_j(t)$, this means:
\be
W_{j{+}1}^\top(t)W_{j{+}1}(t)=W_j(t)W_j^\top(t)~~,~j=1{\ldots}N{-}1
\label{eq:Wj_Wjt}
\ee

\medskip

Regard~$t$ now as fixed, and for every $j=1\ldots{N}$, let:
\be
W_j(t)=U_j\Sigma_{j}V_j^\top
\label{eq:Wj_svd}
\ee
be a singular value decomposition.
That is to say, $U_j$ and~$V_j$ are orthogonal matrices, and~$\Sigma_j$ is a rectangular-diagonal matrix holding non-decreasing, non-negative singular values on its diagonal.
Equation~\ref{eq:Wj_Wjt} implies that for $j=1\ldots{N}{-}1$:
$$V_{j+1}\Sigma_{j+1}^\top\Sigma_{j+1}V_{j+1}^\top=U_j\Sigma_j\Sigma_j^\top{U}_j^\top$$
For a given~$j$, the two sides of the above equation are both orthogonal eigenvalue decompositions of the same matrix.
The square-diagonal matrices~$\Sigma_{j+1}^\top\Sigma_{j+1}$ and~$\Sigma_j\Sigma_j^\top$ are thus the same, up to a possible permutation of diagonal elements (eigenvalues).
However, since by definition $\Sigma_{j+1}$ and~$\Sigma_j$ have non-increasing diagonals, it must hold that~$\Sigma_{j+1}^\top\Sigma_{j+1}=\Sigma_j\Sigma_j^\top$.
Let~$\rho_1{>}\rho_2{>}\cdots{>}\rho_m{\geq}0$ be the distinct eigenvalues, with corresponding multiplicities~$d_1,d_2,\ldots,d_m\in\N$.
We may write:
\be
\Sigma_{j+1}^\top\Sigma_{j+1}=\Sigma_j\Sigma_j^\top=diag(\rho_{1}I_{d_1},\ldots,\rho_{m}I_{d_m})
\label{eq:block_eig}
\ee
where~$I_{d_r}$, $1{\leq}r{\leq}m$, is the identity matrix of size $d_r\times{d}_r$.
Moreover, there exist orthogonal matrices~$O_{j,r}\in\R^{d_r,d_r}$, $1{\leq}r{\leq}m$, such that:
$$U_j=V_{j+1}\cdot{diag}(O_{j,1},\ldots,O_{j,m})$$
$O_{j,r}$~here is simply a matrix changing between orthogonal bases in the eigenspace of~$\rho_r$~--~it maps the basis comprising $V_{j+1}$-columns to that comprising~$U_j$-columns.
Recalling that both $\Sigma_j$ and~$\Sigma_{j+1}$ are rectangular-diagonal, holding only non-negative values, Equation~\ref{eq:block_eig} implies that each of these matrices is equal to $diag(\sqrt{\rho_{1}}{\cdot}I_{d_1},\ldots,\sqrt{\rho_{m}}{\cdot}I_{d_m})$.
Note that the matrices generally do not have the same shape and thus, formally, are not equal to one another.
Nonetheless, in line with our $diag$ notation (see beginning of this subsection), $\Sigma_j$ and~$\Sigma_{j+1}$ may differ from each other only in trailing, zero-valued rows and columns.
By an inductive argument, all the singular value matrices~$\Sigma_1,\Sigma_2,\ldots,\Sigma_N$ (see Equation~\ref{eq:Wj_svd}) are equal up to trailing zero rows and columns.
The fact that $\rho_1\ldots\rho_m$ do not include an index~$j$ in their notation is thus in order, and we may write, for every $j=1\ldots{N}{-}1$:
\beas
W_j(t)&=&U_j\Sigma_{j}V_j^\top \\
&=&V_{j+1}\cdot{diag}(O_{j,1},\ldots,O_{j,m})\cdot \\
&&~~~\qquad{diag}(\sqrt{\rho_{1}}{\cdot}I_{d_1},\ldots,\sqrt{\rho_{m}}{\cdot}I_{d_m})\cdot{V}_j^\top
\eeas
For the~$N$'th weight matrix we have:
\beas
W_N(t)&=&U_N\Sigma_{N}V_N^\top \\
&=&U_N\cdot{diag}(\sqrt{\rho_{1}}{\cdot}I_{d_1},\ldots,\sqrt{\rho_{m}}{\cdot}I_{d_m})\cdot{V}_N^\top
\eeas
Concatenations of weight matrices thus simplify as follows:
\bea
&\prod_{j}^{i=N}W_i(t)\prod_{i=j}^{N}W_i^\top(t)=& 
\label{eq:Wj_train1}\\
&U_N\cdot{diag}\Big((\rho_{1})^{N-j+1}{\cdot}I_{d_1},\ldots,(\rho_{m})^{N-j+1}{\cdot}I_{d_m}\Big)\cdot{U}_N^\top& 
\nonumber\\[1em]
&\prod_{i=1}^{j}W_i^\top(t)\prod_{1}^{i=j}W_i(t)=& 
\label{eq:Wj_train2}\\
&V_1\cdot{diag}\Big((\rho_{1})^j{\cdot}I_{d_1},\ldots,(\rho_{m})^j{\cdot}I_{d_m}\Big)\cdot{V}_1^\top& 
\nonumber\\[1em]
&,~j=1\ldots{N}&
\nonumber
\eea
where we used the orthogonality of~$O_{j,r}$, and the obvious fact that it commutes with~$I_{d_r}$.
Consider Equation~\ref{eq:Wj_train1} with~$j=1$ and Equation~\ref{eq:Wj_train2} with~$j=N$, while recalling that by definition $W_e(t)=\prod^{i=N}_{1}W_j(t)$:
\beas
W_e(t)W_e^\top(t)=U_N{\cdot}diag\Big((\rho_{1})^{N}I_{d_1},\ldots,(\rho_{m})^{N}I_{d_m}\Big){\cdot}U_N^\top \\
W_e^\top(t)W_e(t)=V_1{\cdot}diag\Big((\rho_{1})^{N}I_{d_1},\ldots,(\rho_{m})^{N}I_{d_m}\Big){\cdot}V_1^\top
\eeas
It follows that for every $j=1\ldots{N}$:
\bea
\prod_{j}^{i=N}W_i(t)\prod_{i=j}^{N}W_i^\top(t)=\left[W_e(t)W_e^\top(t)\right]^{\frac{N-j+1}{N}} 
\label{eq:Wj_train1_We}\\
\prod_{i=1}^{j}W_i^\top(t)\prod_{1}^{i=j}W_i(t)=\left[W_e^\top(t)W_e(t)\right]^{\frac{j}{N}} 
\qquad
\label{eq:Wj_train2_We}
\eea
where~$[\cdot]^{\frac{N-j+1}{N}}$ and~$[\cdot]^{\frac{j}{N}}$ stand for fractional power operators defined over positive semidefinite matrices.

\medskip

With Equations~\ref{eq:Wj_train1_We} and~\ref{eq:Wj_train2_We} in place, we are finally in a position to complete the proof.
Returning to Equation~\ref{eq:Wj_gf_L1}, we multiply $\dot{W}_j(t)$ from the left by $\prod^{i=N}_{j+1}W_i(t)$ and from the right by $\prod^{i=j-1}_{1}W_i(t)$, followed by summation over $j=1\ldots{N}$.
This gives:
\beas
\sum\nolimits_{j=1}^{N}\left(\prod\nolimits^{i=N}_{j+1}W_i(t)\right)\dot{W}_j(t)\left(\prod\nolimits^{i=j-1}_{1}W_i(t)\right)= 
\qquad\\
-\eta\lambda\sum\nolimits_{j=1}^{N}\left(\prod\nolimits^{i=N}_{j+1}W_i(t)\right)W_j(t)\left(\prod\nolimits^{i=j-1}_{1}W_i(t)\right) 
~~\\
-\eta\sum\nolimits_{j=1}^{N}\left(\prod\nolimits^{i=N}_{j+1}W_i(t)\prod\nolimits_{i=j+1}^{N}W_i^\top(t)\right)\cdot
~\qquad\qquad\\
\frac{dL^1}{dW}(W_e(t))\cdot\left(\prod\nolimits_{i=1}^{j-1}W_i^\top(t)\prod\nolimits^{i=j-1}_{1}W_i(t)\right)
\eeas
By definition $W_e(t)=\prod^{i=N}_{1}W_j(t)$, so we can substitute the first two lines above:
\beas
\dot{W}_e(t)=-\eta\lambda{N}\cdot{W}_e(t) 
\qquad\qquad\qquad\qquad\qquad\qquad\qquad\\
-\eta\sum_{j=1}^{N}\left(\prod\nolimits^{i=N}_{j+1}W_i(t)\prod\nolimits_{i=j+1}^{N}W_i^\top(t)\right)\cdot
\quad\qquad\qquad\\
\frac{dL^1}{dW}(W_e(t))\cdot\left(\prod\nolimits_{i=1}^{j-1}W_i^\top(t)\prod\nolimits^{i=j-1}_{1}W_i(t)\right)
\eeas
Finally, plugging in the relations in Equations~\ref{eq:Wj_train1_We} and~\ref{eq:Wj_train2_We}, the sought-after result is revealed:
\beas
\dot{W}_e(t)&=&-\eta\lambda{N}\cdot{W}_e(t) \\
&&-\eta\sum_{j=1}^{N}\left[W_e(t)W_e^\top(t)\right]^{\frac{N-j}{N}}\cdot \\
&&\qquad\qquad\frac{dL^1}{dW}(W_e(t))\cdot\left[W_e^\top(t)W_e(t)\right]^{\frac{j-1}{N}}
\eeas

\qed

\bigskip

\begin{lemma}
\label{lemma:coincide}
Let~$I\subset\R$ be a connected interval, and let $f,g:I\to\R$ be differentiable functions.
Suppose that there exists a constant $\alpha\geq0$ for which:
$$\dot{f}(t)+\alpha\cdot{f}(t)=\dot{g}(t)+\alpha\cdot{g}(t)\quad,~\forall{t}\in{I}$$
Then, if~$f$ and~$g$ assume the same value at some $t_0\in{I}$ (interior or boundary), they must coincide along the entire interval, \ie~it must hold that $f(t)=g(t)$ for all~$t\in{I}$.
\end{lemma}
\begin{proof}
Define~$h:=f-g$.
$h$~is a differentiable function from~$I$ to~$\R$, and we have: 
\be
\dot{h}(t)=-\alpha\cdot{h}(t)\quad,~\forall{t}\in{I}
\label{eq:h_de}
\ee
We know that~$h(t_0)=0$ for some~$t_0\in{I}$, and would like to show that~$h(t)=0~\forall{t}\in{I}$.
Assume by contradiction that this is not the case, so there exists~$t_2\in{I}$ for which $h(t_2)\neq0$.
Without loss of generality, suppose that $h(t_2)>0$, and that $t_2>t_0$.
Let~$S$ be the zero set of~$h$, \ie~$S:=\{t\in{I}:h(t)=0\}$.
Since $h$ is continuous in~$I$, $S$~is topologically closed, therefore its intersection with the interval $[t_0,t_2]$ is compact.
Denote by~$t_1$ the maximal element in this intersection, and consider the interval $J:=[t_1,t_2]\subset{I}$.
By construction, $h$~is positive along~$J$, besides on the endpoint~$t_1$ where it assumes the value of zero.
For $t_1<t\leq{t}_2$, we may solve as follows the differential equation of~$h$ (Equation~\ref{eq:h_de}):
$$\frac{\dot{h}(t)}{h(t)}=-\alpha \quad\implies\quad h(t)=\beta{e}^{-\alpha{t}}$$
where~$\beta$ is the positive constant defined by~$h(t_2)=\beta{e}^{-\alpha{t_2}}$.
Since in particular~$h$ is bounded away from zero on~$(t_1,t_2]$, and assumes zero at~$t_1$, we obtain a contradiction to its continuity.
This completes the proof.
\end{proof}

  % END-TO-END UPDATE RULE - VECTORIZED
\subsection{Proof of Claim~\ref{claim:We_gd_vec}} \label{app:proofs:We_gd_vec}

Our proof relies on the \emph{Kronecker product} operation for matrices.
For arbitrary matrices~$A$ and~$B$ of sizes $m_a\times{n}_a$ and $m_b\times{n}_b$ respectively, the Kronecker product $A\odot{B}$ is defined to be the following block matrix:
\be
A\odot{B}:=\begin{bmatrix} 
a_{11}{\cdot}B & \cdots & a_{1n_{a}}{\cdot}B \\[0.25em]
\vdots & \ddots & \vdots \\[0.25em]
a_{m_{a}1}{\cdot}B & \cdots & a_{m_{a}n_{a}}{\cdot}B
\end{bmatrix}
\in\R^{m_{a}m_{b},n_{a}n_{b}}
\label{eq:kron}
\ee
where~$a_{ij}$ stands for the element in row~$i$ and column~$j$ of~$A$.
The Kronecker product admits numerous useful properties.
We will employ the following:
\begin{itemize}
\item
If~$A$ and~$B$ are matrices such that the matrix product $AB$ is defined, then:
\bea
vec(AB)&=&(B^\top\odot{I}_{r_A})\cdot{vec}(A)
\nonumber\\[1mm]
&=&(I_{c_B}\odot{A})\cdot{vec}(B)
\label{eq:kron_vec}
\eea
where~$I_{r_A}$ and~$I_{c_B}$ are the identity matrices whose sizes correspond, respectively, to the number of rows in~$A$ and the number of columns in~$B$.
$vec(\cdot)$~here, as in claim statement, stands for matrix vectorization in column-first order.
\item
If $A_1$, $A_2$, $B_1$ and~$B_2$ are matrices such that the matrix products $A_{1}B_1$ and~$A_{2}B_2$ are defined, then:
\be
(A_1\odot{A}_2)(B_1\odot{B}_2)=(A_{1}B_1)\odot(A_{2}B_2)
\label{eq:kron_prod}
\ee
\item
For any matrices $A$ and~$B$:
\be
(A\odot{B})^\top=A^\top\odot{B}^\top
\label{eq:kron_tpose}
\ee
\item
Equation~\ref{eq:kron_prod} and~\ref{eq:kron_tpose} imply, that if~$A$ and~$B$ are some orthogonal matrices, so is~$A\odot{B}$:
\bea
A^\top=A^{-1}~~\wedge~~B^\top=B^{-1} 
\qquad\qquad\qquad\qquad\nonumber\\
\implies (A\odot{B})^\top=(A\odot{B})^{-1}
\qquad
\label{eq:kron_orth}
\eea
\end{itemize}

\medskip

With the Kronecker product in place, we proceed to the actual proof.
It suffices to show that vectorizing:
$$\sum_{j=1}^N\left[W_e^{(t)}(W_e^{(t)})^\top\right]^\frac{j-1}{N}\cdot
\frac{dL^1}{dW}(W_e^{(t)})\cdot
\left[(W_e^{(t)})^\top{W}_e^{(t)}\right]^\frac{N-j}{N}$$
yields:
$$P_{W_e^{(t)}}\cdot{vec}\left(\frac{dL^1}{dW}(W_e^{(t)})\right)$$
where ${P}_{W_e^{(t)}}$ is the preconditioning matrix defined in claim statement.
For notational conciseness, we hereinafter omit the iteration index~$t$, and simply write $W_e$ instead of $W_e^{(t)}$.

Let~$I_d$ and~$I_k$ be the identity matrices of sizes $d\times{d}$ and~$k\times{k}$ respectively.
Utilizing the properties of the Kronecker product, we have:
\beas
vec\left(\sum_{j=1}^N\left[W_{e}W_e^\top\right]^\frac{j-1}{N}\frac{dL^1}{dW}(W_e)\left[W_e^\top{W}_e\right]^\frac{N-j}{N}\right) 
\quad\qquad\\
=\sum_{j=1}^N\left(I_d\odot\left[W_{e}W_e^\top\right]^\frac{j-1}{N}\right)\cdot
\quad\qquad\qquad\qquad\qquad\qquad\\
\left(\left[W_e^\top{W}_e\right]^\frac{N-j}{N}\odot{I}_k\right)\cdot{vec}\left(\frac{dL^1}{dW}(W_e)\right)
\quad\qquad\\
=\sum_{j=1}^N\left(\left[W_e^\top{W}_e\right]^\frac{N-j}{N}\odot\left[W_{e}W_e^\top\right]^\frac{j-1}{N}\right)vec\left(\frac{dL^1}{dW}(W_e)\right)
\eeas
The first equality here makes use of Equation~\ref{eq:kron_vec}, and the second of Equation~\ref{eq:kron_prod}.
We will show that the matrix:
\be
Q:=\sum_{j=1}^N\left[W_e^\top{W}_e\right]^\frac{N-j}{N}\odot\left[W_{e}W_e^\top\right]^\frac{j-1}{N}
\label{eq:Q}
\ee
meets the characterization of~${P}_{W_e}$, thereby completing the proof.
Let:
$$W_e=UDV^\top$$
be a singular value decomposition, \ie~$U\in\R^{k,k}$ and~$V\in\R^{d,d}$ are orthogonal matrices, and~$D$ is a rectangular-diagonal matrix holding (non-negative) singular values on its diagonal.
Plug this into the definition of~$Q$ (Equation~\ref{eq:Q}):
\beas
Q=\sum_{j=1}^N\left[VD^{\top}DV^{\top}\right]^\frac{N-j}{N}\odot\left[UDD^{\top}U^{\top}\right]^\frac{j-1}{N} 
\qquad\qquad\qquad\\
=\sum_{j=1}^N\left(V\left[D^{\top}D\right]^\frac{N-j}{N}V^{\top}\right)\odot\left(U\left[DD^{\top}\right]^\frac{j-1}{N}U^{\top}\right) 
~~\qquad\\
=\sum_{j=1}^N(V\odot{U})\left(\left[D^{\top}D\right]^\frac{N-j}{N}\odot\left[DD^{\top}\right]^\frac{j-1}{N}\right)(V^{\top}\odot{U}^{\top}) 
~~~\\
=(V\odot{U})\left(\sum_{j=1}^N\left[D^{\top}D\right]^\frac{N-j}{N}\odot\left[DD^{\top}\right]^\frac{j-1}{N}\right)(V\odot{U})^\top
\quad
\eeas
The third equality here is based on the relation in Equation~\ref{eq:kron_prod}, and the last equality is based on Equation~\ref{eq:kron_tpose}.
Denoting:
\bea
O&:=&V\odot{U} 
\label{eq:O}\\
\Lambda&:=&\sum_{j=1}^N\left[D^{\top}D\right]^\frac{N-j}{N}\odot\left[DD^{\top}\right]^\frac{j-1}{N}
\label{eq:Lambda}
\eea
we have:
\be
Q=O\Lambda{O}^\top
\label{eq:Q_evd}
\ee
Now, since by definition $U$ and~$V$ are orthogonal, $O$~is orthogonal as well (follows from the relation in Equation~\ref{eq:kron_orth}).
Additionally, the fact that~$D$ is rectangular-diagonal implies that the square matrix $\Lambda$ is also diagonal.
Equation~\ref{eq:Q_evd} is thus an orthogonal eigenvalue decomposition of~$Q$.
Finally, denote the columns of~$U$ (left singular vectors of~$W_e$) by $\uu_1\ldots\uu_k$, those of~$V$ (right singular 
vectors of~$W_e$) by $\vv_1\ldots\vv_d$, and the diagonal elements of~$D$ (singular values of~$W_e$) by $\sigma_1\ldots\sigma_{\max\{k,d\}}$ (by definition $\sigma_r=0$ if $r>\min\{k,d\}$).
The definitions in Equations~\ref{eq:O} and~\ref{eq:Lambda} imply that the columns of~$O$ are:
$$vec(\uu_r\vv_{r'}^\top)\quad,r=1\ldots{k}~,~r'=1\ldots{d}$$
with corresponding diagonal elements in~$\Lambda$ being:
$$\sum\nolimits_{j=1}^{N}\sigma_r^{2\frac{N-j}{N}}\sigma_{r'}^{2\frac{j-1}{N}}\quad,r=1\ldots{k}~,~r'=1\ldots{d}$$
We conclude that~$Q$ indeed meets the characterization of~$P_{W_e}$ in claim statement.
This completes the proof.

\qed

  % END-TO-END UPDATE RULE - SINGLE OUTPUT
\subsection{Proof of Claim~\ref{claim:We_gd_single}} \label{app:proofs:We_gd_single}

We disregard the trivial case~$N=1$, as well as the scenario $W_e^{(t)}=0$ (both lead Equations~\ref{eq:We_gd} and~\ref{eq:We_gd_single} to equate).
Omitting the iteration index~$t$ from our notation, it suffices to show that:
\bea
\sum_{j=1}^N\left[W_{e}W_e^\top\right]^\frac{j-1}{N}\cdot\frac{dL^1}{dW}(W_e)\cdot\left[W_e^\top{W}_e\right]^\frac{N-j}{N}=
\quad
\label{eq:We_gd_terms}\\
\norm{W_e}_{2}^{2-\frac{2}{N}}\left(\tfrac{dL^1}{dW}(W_e)+(N-1)Pr_{W_e}\big\{\tfrac{dL^1}{dW}(W_e)\big\}\right)
\nonumber
\eea
where $Pr_{W_e}\{\cdot\}$ is the projection operator defined in claim statement (Equation~\ref{eq:proj}), and we recall that by assumption~$k=1$ ($W_e\in\R^{1,d}$).
$\left[W_{e}W_e^\top\right]^\frac{j-1}{N}$~is a scalar, equal to~$\norm{W_e}_2^{2\frac{j-1}{N}}$ for every $j=1\ldots{N}$.
$\left[W_e^\top{W}_e\right]^\frac{N-j}{N}$~on the other hand is a $d\times{d}$ matrix, by definition equal to identity for $j=N$, and otherwise, for $j=1\ldots{N}-1$, it is equal to $\norm{W_e}_2^{2\frac{N-j}{N}}\left(\nicefrac{W_e}{\norm{W_e}_2}\right)^\top\left(\nicefrac{W_e}{\norm{W_e}_2}\right)$.
Plugging these equalities into the first line of Equation~\ref{eq:We_gd_terms} gives:
\beas
\sum_{j=1}^N\left[W_{e}W_e^\top\right]^\frac{j-1}{N}\frac{dL^1}{dW}(W_e)\left[W_e^\top{W}_e\right]^\frac{N-j}{N}= 
\qquad\\
\sum_{j=1}^{N-1}\norm{W_e}_2^{2\frac{j-1}{N}}\frac{dL^1}{dW}(W_e)\norm{W_e}_2^{2\frac{N-j}{N}}\left(\tfrac{W_e}{\norm{W_e}_2}\right)^\top\left(\tfrac{W_e}{\norm{W_e}_2}\right) \\
+\norm{W_e}_2^{2\frac{N-1}{N}}\cdot\frac{dL^1}{dW}(W_e)= 
\qquad\qquad\qquad\\[1mm]
(N-1)\norm{W_e}_2^{2\frac{N-1}{N}}\frac{dL^1}{dW}(W_e)\left(\tfrac{W_e}{\norm{W_e}_2}\right)^\top\left(\tfrac{W_e}{\norm{W_e}_2}\right) 
\quad\\
+\norm{W_e}_2^{2\frac{N-1}{N}}\cdot\frac{dL^1}{dW}(W_e)
~~\quad\qquad\qquad\qquad
\eeas
The latter expression is precisely the second line of Equation~\ref{eq:We_gd_terms}, thus our proof is complete.

\qed

  % END-TO-END STEP NOT GRADIENT
\subsection{Proof of Theorem~\ref{theorem:impossible}} \label{app:proofs:impossible}

Our proof relies on elementary differential geometry: curves, arc length and line integrals (see Chapters~8 and~9 in~\citet{buck2003advanced}).

Let~$\U\subset\R^{1,d}$ be a neighborhood of $W=0$ (\ie~an open set that includes this point) on which $\frac{dL^1}{dW}$ is continuous ($\U$~exists by assumption).
It is not difficult to see that $F(\cdot)$ (Equation~\ref{eq:F}) is continuous on~$\U$ as well.
The strategy of our proof will be to show that $F(\cdot)$ does not admit the \emph{gradient theorem} (also known as the \emph{fundamental theorem for line integrals}).
According to the theorem, if $h:\U\to\R$ is a continuously differentiable function, and $\Gamma$ is a piecewise smooth curve lying in~$\U$ with start-point~$\gamma_s$ and end-point~$\gamma_e$, then:
$$\int_\Gamma\frac{dh}{dW}=h(\gamma_e)-h(\gamma_s)$$
In words, the line integral of the gradient of~$h$ over~$\Gamma$, is equal to the difference between the value taken by~$h$ at the end-point of~$\Gamma$, and that taken at the start-point.
A direct implication of the theorem is that if $\Gamma$ is closed ($\gamma_e=\gamma_s$), the line integral vanishes:
$$\oint_{\Gamma}\frac{dh}{dW}=0$$
We conclude that if~$F(\cdot)$ is the gradient field of some function, its line integral over any closed (piecewise smooth) curve lying in~$\U$ must vanish.
We will show that this is not the case.

\medskip

For notational conciseness we hereinafter identify $\R^{1,d}$ and~$\R^d$, so in particular $\U$ is now a subset of~$\R^d$.
To further simplify, we omit the subindex from the Euclidean norm, writing~$\norm{\cdot}$ instead of~$\norm{\cdot}_2$.
Given an arbitrary continuous vector field $\phi:\U\to\R^d$, we define a respective (continuous) vector field as follows:
\bea
F_\phi:\U\to\R^d 
\,~\qquad\qquad\qquad\qquad
\nonumber\\[1em]
F_\phi(\w)= 
\quad\qquad\qquad\qquad\qquad
\label{eq:F_phi}\\[0.25em]
\begin{cases}
\hspace{-0.5mm}\norm{\w}^{2{-}\frac{2}{N}}\hspace{-1mm}\left(\phi(\w){+}(N{-}1)\hspace{-0.5mm}\inprod{\phi(\w)}{\frac{\w}{\norm{\w}}}\hspace{-0.5mm}\frac{\w}{\norm{\w}}\right) 
& \hspace{-2.5mm},\w{\neq}\0 \\
\qquad\qquad\qquad\qquad\qquad\0 
& \hspace{-2.5mm},\w{=}\0
\end{cases}
\nonumber
\eea
Notice that for $\phi=\frac{dL^1}{dW}$, we get exactly the vector field $F(\cdot)$ defined in theorem statement (Equation~\ref{eq:F})~--~the subject of our inquiry.
As an operator on (continuous) vector fields, the mapping $\phi\mapsto{F}_\phi$ is linear.\note{
For any $\phi_1,\phi_2:\U\to\R^d$ and $c\in\R$, it holds that $F_{\phi_1+\phi_2}=F_{\phi_1}+F_{\phi_2}$ and $F_{c\cdot\phi_1}=c\cdot{F}_{\phi_1}$.
}
This, along with the linearity of line integrals, imply that for any piecewise smooth curve~$\Gamma$ lying in~$\U$, the functional $\phi\mapsto\int_{\Gamma}F_\phi$, a mapping of (continuous) vector fields to scalars, is linear.
Lemma~\ref{lemma:line_int_ub} below provides an upper bound on this linear functional in terms of the length of~$\Gamma$, its maximal distance from origin, and the maximal norm $\phi$ takes on it.

In light of the above, to show that $F(\cdot)$ contradicts the gradient theorem, thereby completing the proof, it suffices to find a closed (piecewise smooth) curve~$\Gamma$ for which the linear functional $\phi\mapsto\oint_\Gamma{F}_\phi$ does not vanish at $\phi=\frac{dL^1}{dW}$.
By assumption $\frac{dL^1}{dW}(W{=}0)\neq0$, and so we may define the unit vector in the direction of~$\frac{dL^1}{dW}(W{=}0)$:
\be
\e:=\frac{\frac{dL^1}{dW}(W{=}0)}{\norm{\frac{dL^1}{dW}(W{=}0)}}\in\R^d
\label{eq:e}
\ee
Let~$R$ be a positive constant small enough such that the Euclidean ball of radius~$R$ around the origin is contained in~$\U$.
Let~$r$ be a positive constant smaller than~$R$.
Define~$\Gamma_{r,R}$ to be a curve as follows (see illustration in Figure~\ref{fig:curve}):\note{
The proof would have been slightly simplified had we used a curve that passes directly through the origin.
We avoid this in order to emphasize that the result is not driven by some point-wise singularity (the origin received special treatment in the definition of $F(\cdot)$~--~see Equations~\ref{eq:F} and~\ref{eq:proj}).
}
\bea
\Gamma_{r,R}:=\Gamma_{r,R}^1~\to~\Gamma_{r,R}^2~\to~\Gamma_{r,R}^3~\to~\Gamma_{r,R}^4
\label{eq:curve}
\eea
where:
\begin{itemize}
\item $\Gamma_{r,R}^1$ is the line segment from~$-R\cdot\e$ to~$-r\cdot\e$.
\item $\Gamma_{r,R}^2$ is a geodesic on the sphere of radius~$r$, starting from~$-r\cdot\e$ and ending at~$r\cdot\e$.
\item $\Gamma_{r,R}^3$ is the line segment from~$r\cdot\e$ to~$R\cdot\e$.
\item $\Gamma_{r,R}^4$ is a geodesic on the sphere of radius~$R$, starting from~$R\cdot\e$ and ending at~$-R\cdot\e$.
\end{itemize}
$\Gamma_{r,R}$ is a piecewise smooth, closed curve that fully lies within~$\U$.
Consider the linear functional it induces: $\phi\mapsto\oint_{\Gamma_{r,R}}{F}_\phi$.
We will evaluate this functional on $\phi=\frac{dL^1}{dW}$.
To do so, we decompose the latter as follows:
\be
\tfrac{dL^1}{dW}(\cdot)=c\cdot\e(\cdot)+\xi(\cdot)
\label{eq:grad_decomp}
\ee
where:
\begin{itemize}
\item 
$c$ is a scalar equal to $\|\tfrac{dL^1}{dW}(W{=}0)\|$.
\item 
$\e(\cdot)$ is a vector field returning the constant~$\e$ (Equation~\ref{eq:e}).
\item 
$\xi(\cdot)$ is a vector field returning the values of $\frac{dL^1}{dW}(\cdot)$ shifted by the constant $-\tfrac{dL^1}{dW}(W{=}0)$.
It is continuous on~$\U$ and vanishes at the origin.
\end{itemize}
Applying Lemma~\ref{lemma:line_int_ub} to $\xi$ over~$\Gamma_{r,R}$ gives:
\beas
\abs{\oint_{\Gamma_{r,R}}{F}_\xi}\leq{N}\cdot{len}(\Gamma_{r,R})\cdot\max_{\gamma\in\Gamma_{r,R}}\norm{\gamma}^{2{-}\frac{2}{N}}\cdot\max_{\gamma\in\Gamma_{r,R}}\norm{\xi(\gamma)} \\
=N\cdot(\pi{r}+\pi{R}+2(R-r))\cdot{R}^{2{-}\frac{2}{N}}\cdot\max_{\gamma\in\Gamma_{r,R}}\norm{\xi(\gamma)} \\
\leq{N}\cdot2\pi\cdot{R}^{3{-}\frac{2}{N}}\cdot\max_{\gamma\in\Gamma_{r,R}}\norm{\xi(\gamma)}
\,~\qquad\qquad\qquad\qquad
\eeas
On the other hand, by Lemma~\ref{lemma:line_int_comp}:
$$\oint_{\Gamma_{r,R}}{F}_\e=\left(\frac{2N}{3-\nicefrac{2}{N}}-2\right)\left(R^{3{-}\frac{2}{N}}-r^{3{-}\frac{2}{N}}\right)$$
The linearity of the functional $\phi\mapsto\oint_{\Gamma_{r,R}}{F}_\phi$, along with Equation~\ref{eq:grad_decomp}, then imply:
\beas
\oint_{\Gamma_{r,R}}{F}_{\frac{dL^1}{dW}}&=&c\cdot\oint_{\Gamma_{r,R}}{F}_\e+\oint_{\Gamma_{r,R}}{F}_\xi \\
&\geq&{c}\cdot\left(\frac{2N}{3-\nicefrac{2}{N}}-2\right)\left(R^{3{-}\frac{2}{N}}-r^{3{-}\frac{2}{N}}\right) \\
&&~-N\cdot2\pi\cdot{R}^{3{-}\frac{2}{N}}\cdot\max_{\gamma\in\Gamma_{r,R}}\norm{\xi(\gamma)}
\eeas
We will show that for proper choices of~$R$ and~$r$, the lower bound above is positive.
$\Gamma_{r,R}$~will then be a piecewise smooth closed curve lying in~$\U$, for which the functional $\phi\mapsto\oint_{\Gamma_{r,R}}{F}_\phi$ does not vanish at $\phi=\frac{dL^1}{dW}$.
As stated, this will imply that $F(\cdot)$ violates the gradient theorem, thereby concluding our proof.

\medskip

All that is left is to affirm that the expression:
\beas
&c\cdot\left(\frac{2N}{3-\nicefrac{2}{N}}-2\right)\left(R^{3{-}\frac{2}{N}}-r^{3{-}\frac{2}{N}}\right)& \\
&-N\cdot2\pi\cdot{R}^{3{-}\frac{2}{N}}\cdot\max_{\gamma\in\Gamma_{r,R}}\norm{\xi(\gamma)}&
\eeas
can indeed be made positive with proper choices of~$R$ and~$r$.
Recall that:
\begin{itemize}
\item $N>2$ by assumption; implies $\frac{2N}{3-\nicefrac{2}{N}}-2>0$.
\item $R$ is any positive constant small enough such that the ball of radius~$R$ around the origin is contained in~$\U$.
\item $r$ is any positive constant smaller than~$R$.
\item $\Gamma_{r,R}$ is a curve whose points are all within distance~$R$ from the origin.
\item $c=\|\tfrac{dL^1}{dW}(W{=}0)\|$~--~positive by assumption.
\item $\xi(\cdot)$ is a vector field that is continuous on~$\U$ and vanishes at the origin.
\end{itemize}
The following procedure gives $R$ and~$r$ as required:
\begin{itemize}
\item Set $r$ to follow $R$ such that: $r^{3{-}\frac{2}{N}}=0.5\cdot{R}^{3{-}\frac{2}{N}}$.
\item Choose $\epsilon>0$ for which $0.5c\left(\frac{2N}{3-\frac{2}{N}}{-}2\right)-2\pi{N}\epsilon>0$.
\item Set $R$ to be small enough such that $\norm{\xi(\w)}\leq\epsilon$ for any point $\w$ within distance~$R$ from the origin.
\end{itemize}
The proof is complete.

\qed

\bigskip

\begin{lemma}
\label{lemma:line_int_ub}
Let $\phi:\U\to\R^d$ be a continuous vector field, and let $\Gamma$ be a piecewise smooth curve lying in~$\U$.
Consider the (continuous) vector field $F_\phi:\U\to\R^d$ defined in Equation~\ref{eq:F_phi}.
The line integral of the latter over~$\Gamma$ is bounded as follows:
$$\abs{\int_\Gamma{F}_\phi}\leq{N}\cdot{len}(\Gamma)\cdot\max_{\gamma\in\Gamma}\norm{\gamma}^{2{-}\frac{2}{N}}\cdot\max_{\gamma\in\Gamma}\norm{\phi(\gamma)}$$
where $len(\Gamma)$ is the arc length of~$\Gamma$, and $\gamma\in\Gamma$ refers to a point lying on the curve.
\end{lemma}
\begin{proof}
We begin by noting that the use of $\max$ (as opposed to $\sup$) in stated upper bound is appropriate, since under our definition of a curve (adopted from~\citet{buck2003advanced}), points lying on it constitute a compact set.
This subtlety is of little importance~--~one may as well replace $\max$ by $\sup$, and the lemma would still serve its purpose.

It is not difficult to see that for any $\w\in\U$, $\w\neq0$:
\beas
\norm{F_\phi(\w)}
{=}\norm{\w}^{2{-}\frac{2}{N}}\hspace{-1mm}\norm{\phi(\w){+}(N{-}1)\hspace{-0.5mm}\inprod{\phi(\w)}{\frac{\w}{\norm{\w}}}\hspace{-0.5mm}\frac{\w}{\norm{\w}}} 
\\
\leq\norm{\w}^{2{-}\frac{2}{N}}\hspace{-1mm}\left(\norm{\phi(\w)}{+}(N{-}1)\hspace{-0.5mm}\abs{\inprod{\phi(\w)}{\frac{\w}{\norm{\w}}}}\hspace{-0.5mm}{\cdot}\norm{\frac{\w}{\norm{\w}}}\right) 
\\
=\norm{\w}^{2{-}\frac{2}{N}}\hspace{-1mm}\left(\norm{\phi(\w)}{+}(N{-}1)\hspace{-0.5mm}\abs{\inprod{\phi(\w)}{\frac{\w}{\norm{\w}}}}\right) 
\,~~\quad\qquad\\
\leq\norm{\w}^{2{-}\frac{2}{N}}\hspace{-1mm}\left(\norm{\phi(\w)}{+}(N{-}1)\hspace{-0.5mm}\norm{\phi(\w)}\right) 
\qquad\qquad\qquad\qquad\\[2mm]
\leq{N}\norm{\w}^{2{-}\frac{2}{N}}\norm{\phi(\w)}
~~\qquad\qquad\qquad\qquad\qquad\qquad\qquad
\eeas
Trivially, $\norm{F_\phi(\w)}\leq{N}\norm{\w}^{2{-}\frac{2}{N}}\norm{\phi(\w)}$ holds for $\w{=}0$ as well.
The sought-after result now follows from the properties of line integrals:
\beas
\abs{\int_\Gamma{F}_\phi}
\leq\int_\Gamma\norm{F_\phi}
\leq\int_\Gamma{N}\norm{\w}^{2{-}\frac{2}{N}}\norm{\phi(\w)} 
\qquad\\
\leq{N}\cdot{len}(\Gamma)\cdot\max_{\gamma\in\Gamma}\norm{\gamma}^{2{-}\frac{2}{N}}\cdot\max_{\gamma\in\Gamma}\norm{\phi(\gamma)}
\eeas
\end{proof}

\begin{lemma}
\label{lemma:line_int_comp}
Let $\e$ be a unit vector, let $\Gamma_{r,R}$ be a piecewise smooth closed curve as specified in Equation~\ref{eq:curve} and the text thereafter, and let $\phi\mapsto{F}_\phi$ be the operator on continuous vector fields defined by Equation~\ref{eq:F_phi}.
Overloading notation by regarding~$\e(\cdot)\equiv\e$ as a constant vector field, it holds that:
$$\oint_{\Gamma_{r,R}}{F}_\e=\left(\frac{2N}{3-\nicefrac{2}{N}}-2\right)\left(R^{3{-}\frac{2}{N}}-r^{3{-}\frac{2}{N}}\right)$$
\end{lemma}
\begin{proof}
We compute the line integral by decomposing~$\Gamma_{r,R}$ into its smooth components $\Gamma_{r,R}^1\ldots\Gamma_{r,R}^4$:
\be
\oint_{\Gamma_{r,R}}\hspace{-3mm}F_\e~=~\int_{\Gamma_{r,R}^1}\hspace{-3mm}F_\e+\int_{\Gamma_{r,R}^2}\hspace{-3mm}F_\e+\int_{\Gamma_{r,R}^3}\hspace{-3mm}F_\e+\int_{\Gamma_{r,R}^4}\hspace{-3mm}F_\e
\label{eq:line_int_comp_parts}
\ee

Starting from~$\Gamma_{r,R}^1$, notice that for every point $\w$ lying on this curve:
$\langle\e,\frac{\w}{\norm{\w}}\rangle\frac{\w}{\norm{\w}}=\e$.
Therefore:
$$\int_{\Gamma_{r,R}^1}\hspace{-3mm}F_\e
=\int_{\Gamma_{r,R}^1}\hspace{-3mm}\norm{\w}^{2-\frac{2}{N}}\left(\e{+}(N{-}1)\e\right)
=N\int_{\Gamma_{r,R}^1}\hspace{-3mm}\norm{\w}^{2-\frac{2}{N}}\e$$
The line integral on the right translates into a simple univariate integral:
\beas
\int_{\Gamma_{r,R}^1}\hspace{-3mm}\norm{\w}^{2-\frac{2}{N}}\e=\int_{-R}^{-r}\abs{\rho}^{2-\frac{2}{N}}d\rho=\int_{r}^{R}\rho^{2-\frac{2}{N}}d\rho \\
=\frac{1}{3-\nicefrac{2}{N}}\left(R^{3-\frac{2}{N}}-r^{3-\frac{2}{N}}\right)
\qquad
\eeas
We thus have:
\be
\int_{\Gamma_{r,R}^1}\hspace{-3mm}F_\e=\frac{N}{3-\nicefrac{2}{N}}\left(R^{3-\frac{2}{N}}-r^{3-\frac{2}{N}}\right)
\label{eq:line_int_comp_part1}
\ee

Turning to~$\Gamma_{r,R}^2$, note that for any point $\w$ along this curve $\norm{\w}^{2-\frac{2}{N}}=r^{2-\frac{2}{N}}$, and $\frac{\w}{\norm{\w}}$ is perpendicular to the direction of motion.
This implies:
$$\int_{\Gamma_{r,R}^2}\hspace{-3mm}F_\e=r^{2-\frac{2}{N}}\int_{\Gamma_{r,R}^2}\hspace{-3mm}\e$$
The line integral $\int_{\Gamma_{r,R}^2}\hspace{-2mm}\e$ is simply equal to the progress $\Gamma_{r,R}^2$ makes in the direction of~$\e$, which is~$2r$.
Accordingly:
\be
\int_{\Gamma_{r,R}^2}\hspace{-3mm}F_\e=r^{2-\frac{2}{N}}\cdot2r=2r^{3-\frac{2}{N}}
\label{eq:line_int_comp_part2}
\ee

As for $\Gamma_{r,R}^3$ and~$\Gamma_{r,R}^4$, their line integrals may be computed similarly to those of $\Gamma_{r,R}^1$ and~$\Gamma_{r,R}^2$ respectively.
Such computations yield:
\bea
&&\int_{\Gamma_{r,R}^3}\hspace{-3mm}F_\e=\frac{N}{3-\nicefrac{2}{N}}\left(R^{3-\frac{2}{N}}-r^{3-\frac{2}{N}}\right)
\label{eq:line_int_comp_part3} \\
&&\int_{\Gamma_{r,R}^4}\hspace{-3mm}F_\e=-2R^{3-\frac{2}{N}}
\label{eq:line_int_comp_part4}
\eea

Combining Equation~\ref{eq:line_int_comp_parts} with Equations~\ref{eq:line_int_comp_part1}, \ref{eq:line_int_comp_part2}, \ref{eq:line_int_comp_part3} and~\ref{eq:line_int_comp_part4}, we obtain the desired result.
\end{proof}

% A CONCRETE ACCELERATION BOUND
\section{A Concrete Acceleration Bound} \label{app:acceleration_bound}

In Section~\ref{sec:acceleration} we illustrated qualitatively, on a family of very simple hypothetical learning problems, the potential of overparameterization (use of depth-$N$ linear network in place of classic linear model) to accelerate optimization.
In this appendix we demonstrate how the illustration can be made formal, by considering a special case and deriving a concrete bound on the acceleration.

In the context of Section~\ref{sec:acceleration}, we will treat the setting of $p=4$ ($\ell_4$~loss) and $N=2$ (depth-$2$ network).
We will also assume, in accordance with the problem being ill-conditioned~--~$y_1{\gg}{y}_2$, that initialization values are ill-conditioned as well, and in particular~$\epsilon_1/\epsilon_2\approx{y}_1/y_2$, where $\epsilon_i:=|w_i^{(0)}|$.
An additional assumption we make is that~$y_2$ is on the order of~$1$, and thus the near-zero initialization of $w_1$ and~$w_2$ implies $y_2\gg\epsilon_1,\epsilon_2$.
Finally, we assume that $\epsilon_{1}y_1\gg1$.

As shown in Section~\ref{sec:acceleration}, under gradient descent, $w_1$ and~$w_2$ move independently, and to prevent divergence, the learning rate must satisfy $\eta<\min\{2/y_1^{p-2},2/y_2^{p-2}\}$.
In our setting, this translates to (GD below stands for gradient descent):
\vspace{-2mm}
\be
\eta^{GD}<{2}/{y_1^2}
\label{eq:eta_gd}
\ee
\vspace{-6mm}\\
For~$w_2$, the optimal learning rate (convergence in a single step) is~$1/y_2^2$, and the constraint above will lead to very slow convergence (see Equation~\ref{eq:illus_gd_delta} and its surrounding text).

Suppose now that we optimize via overparameterization, \ie~with the update rule in Equation~\ref{eq:We_gd_single} (single output).
In our particular setting (recall, in addition to the above, that we omitted weight decay for simplicity~--~$\lambda=0$), this update rule translates to:
\vspace{-1mm}
{\small
\bea
[w_1^{(t+1)},w_2^{(t+1)}]^\top\mapsfrom[w_1^{(t)},w_2^{(t)}]^\top
~~\qquad\qquad\qquad\qquad\qquad
\label{eq:illus_op}\\
-\eta\Big((w_1^{(t)})^2+(w_2^{(t)})^2\Big)^{1/2}\cdot[(w_1^{(t)}-y_1)^3,(w_2^{(t)}-y_2)^3]^\top
\nonumber\\
-\eta\Big((w_1^{(t)})^2+(w_2^{(t)})^2\Big)^{-1/2}
\qquad\qquad\qquad\qquad\qquad\qquad
\nonumber\\
\cdot\big(w_1^{(t)}(w_1^{(t)}-y_1)^3+w_2^{(t)}(w_2^{(t)}-y_2)^3\big)\cdot[w_1^{(t)},w_2^{(t)}]^\top
\nonumber
\eea
}
\vspace{-5mm}\\
For the first iteration ($t=0$), replacing $\epsilon_i:=|w_i^{(0)}|$, while recalling that $y_1\gg{y_2}\gg\epsilon_1\gg\epsilon_2$, we obtain:
\vspace{-1mm}
\beas
[w_1^{(1)},w_2^{(1)}]^\top&\approx&\eta\cdot\epsilon_1\cdot[y_1^3,y_2^3]^\top+\eta\cdot{y}_1^3\cdot[\epsilon_1,\epsilon_2]^\top 
\\[1mm]
&=&\eta\cdot[2\epsilon_{1}y_1^3,\epsilon_{1}y_2^3+\epsilon_{2}y_1^3]^\top
\eeas
\vspace{-5mm}\\
Set $\eta=1/2\epsilon_{1}y_1^2$.
Then $w_1^{(1)}\approx{y}_1$ and $w_2^{(1)}\approx{y}_2^3/2y_1^2+\epsilon_{2}y_1/2\epsilon_1$.
Our assumptions thus far ($y_1\gg{y}_2$ and $\epsilon_1\gg\epsilon_2$) imply $w_1^{(1)}\gg{w}_2^{(1)}$.
Moreover, since $\epsilon_2/\epsilon_1\approx{y}_2/y_1$, it holds that $w_2^{(1)}\in\OO(y_2)=\OO(1)$.
Taking all of this into account, the second iteration ($t=1$) of the overparameterized update rule (Equation~\ref{eq:illus_op}) becomes:
\vspace{-1mm}
\beas
[w_1^{(2)},w_2^{(2)}]^\top\approx[y_1,w_2^{(1)}]^\top
\quad\qquad\qquad\qquad\qquad\qquad\qquad\\
-\frac{1}{2\epsilon_{1}y_1}[(w_1^{(1)}-y_1)^3,(w_2^{(1)}-y_2)^3]^\top 
~~\qquad\qquad\\
-\frac{y_1(w_1^{(1)}-y_1)^3+w_2^{(1)}(w_2^{(1)}-y_2)^3}{2\epsilon_{1}y_1^3}[y_1,w_2^{(1)}]^\top
\\
\approx[y_1,w_2^{(1)}-1/2\epsilon_{1}y_1\cdot(w_2^{(1)}-y_2)^3]^\top
~~\quad\qquad\qquad
\eeas
\vspace{-5mm}\\
In words, $w_1$~will stay approximately equal to~$y_1$, whereas $w_2$~will take a step that corresponds to gradient descent with learning rate (OP below stands for overparameterization):
\vspace{-2mm}
\be
\eta^{OP}:={1}/{2\epsilon_{1}y_1}
\label{eq:eta_op}
\ee
\vspace{-5mm}\\
By assumption $\epsilon_{1}y_1{\gg}1$ and~$y_2{\in}\OO(1)$, thus $\eta^{OP}{<}2/y_2^2$, meaning that $w_2$~will remain on the order of~$y_2$ (or less).
An inductive argument can therefore be applied, and our observation regarding the second iteration ($t=1$) continues to hold throughout~--~$w_1$ is (approximately) fixed at~$y_1$, and $w_2$ follows steps that correspond to gradient descent with learning rate~$\eta^{OP}$. 

To summarize our findings, we have shown that while standard gradient descent limits~$w_2$ with a learning rate~$\eta^{GD}$ that is at most~$2/y_1^2$ (Equation~\ref{eq:eta_gd}), overparameterization can be adjusted to induce on~$w_2$ an implicit gradient descent scheme with learning rate~$\eta^{OP}=1/2\epsilon_{1}y_1$ (Equation~\ref{eq:eta_op}), all while admitting immediate (single-step) convergence for~$w_1$.
Since both $\eta^{GD}$ and~$\eta^{OP}$ are well below~$1/y_2^2$, we obtain acceleration by at least $\eta^{OP}/\eta^{GD}>y_1/4\epsilon_1$ (we remind the reader that $y_1\gg1$ is the target value of~$w_1$, and $\epsilon_1\ll1$ is the magnitude of its initialization).

% IMPLEMENTATION DETAILS
\section{Implementation Details} \label{app:impl}

Below we provide implementation details omitted from our experimental report (Section~\ref{sec:exp}).

  % LINEAR NEURAL NETWORKS
\subsection{Linear Neural Networks} \label{app:impl:lnn}

The details hereafter apply to all of our experiments besides that on the convolutional network (Figure~\ref{fig:exp_resnet_cnn}-right).

In accordance with our theoretical setup (Section~\ref{sec:lnn}), evaluated linear networks did not include bias terms, only weight matrices.
The latter were initialized to small values, drawn i.i.d.~from a Gaussian distribution with mean zero and standard deviation~$0.01$.
The only exception to this was the setting of identity initialization (Figure~\ref{fig:exp_resnet_cnn}-left), in which an offset of~$1$ was added to the diagonal elements of each weight matrix (including those that are not square).

When applying a grid search over learning rates, the values $\{10^{-5},5\cdot10^{-5},\ldots,10^{-1},5\cdot10^{-1}\}$ were tried.
We note that in the case of depth-$8$ network with standard near-zero initialization (Figure~\ref{fig:exp_resnet_cnn}-left), all learning rates led either to divergence, or to a failure to converge (vanishing gradients).

For computing optimal $\ell_2$~loss (used as an offset in respective convergence plots), we simply solved, in closed form, the corresponding least squares problem.
For the optimal $\ell_4$~loss, we used \texttt{scipy.optimize.minimize}~--~a numerical optimizer built into SciPy~\cite{scipy}, with the default method of BFGS~\cite{jorge}.

  % CONVOLUTIONAL NETWORK
\subsection{Convolutional Network} \label{app:impl:cnn}

For the experiment on TensorFlow's MNIST convolutional network tutorial, we simply downloaded the code,\note{
\url{https://github.com/tensorflow/models/tree/master/tutorials/image/mnist}
} and introduced two minor changes:
\vspace{-3mm}
\begin{itemize}
\item Hidden dense layer: $3136{\times}512$ weight matrix replaced by multiplication of $3136{\times}512$ and~$512{\times}512$ matrices.
\vspace{-6mm}
\item Output layer: $512{\times}10$ weight matrix replaced by multiplication of $512{\times}10$ and~$10{\times}10$ matrices.
\end{itemize}
\vspace{-3mm}
The newly introduced weight matrices were initialized in the same way as their predecessors (random Gaussian distribution with mean zero and standard deviation~$0.1$).
Besides the above, no change was made.
An addition of roughly~$250K$ parameters to a $1.6M$-parameter model gave the speedup presented in Figure~\ref{fig:exp_resnet_cnn}-right.

\medskip

To rule out the possibility of the speedup resulting from suboptimal learning rates, we reran the experiment with grid search over the latter.
The learning rate hardcoded into the tutorial follows an exponentially decaying schedule, with base value~$10^{-2}$.
For both the original and overparameterized models, training was run multiple times, with the base value varying in $\{10^{-5},5\cdot10^{-5},\ldots,10^{-1},5\cdot10^{-1}\}$.
We chose, for each model separately, the configuration giving fastest convergence, and then compared the models one against the other.
The observed gap in convergence rates was similar to that in Figure~\ref{fig:exp_resnet_cnn}-right.

An additional point we set out to examine, is the sensitivity of the speedup to initialization of overparameterized layers.
For this purpose, we retrained the overparameterized model multiple times, varying in $\{10^{-3},5\cdot10^{-3},\ldots,10^{-1},5\cdot10^{-1}\}$ the standard deviation of the Gaussian distribution initializing overparameterized layers (as stated above, this standard deviation was originally set to~$10^{-1}$).
Convergence rates across the different runs were almost identical.
In particular, they were all orders of magnitude faster than the convergence rate of the baseline, non-overparameterized~model.

\end{document}